%% file: main.tex
\documentclass{article}

\usepackage[main,final]{neurips_2026}

\usepackage[utf8]{inputenc}
\usepackage[T1]{fontenc}
\usepackage{hyperref}
\hypersetup{hidelinks}
\usepackage{url}
\usepackage{booktabs}
\usepackage{amsfonts}
\usepackage{nicefrac}
\usepackage{microtype}
\usepackage{xcolor}

\usepackage{amsmath}
\usepackage{amssymb}
\usepackage{amsthm}
\usepackage{graphicx}
\usepackage{subcaption}
\usepackage{algorithm}
\usepackage{algorithmic}
\usepackage{enumitem}
\usepackage{multirow}
\usepackage{wrapfig}

\newtheorem{theorem}{Theorem}

\newcommand{\Expect}{\mathbb{E}}

\newcommand{\KG}{\mathcal{K}}
\newcommand{\Gcap}{{G_{\mathrm{cap}}}}
\newcommand{\Gtask}{{G_{\mathrm{task}}}}
\newcommand{\Gexp}{{G_{\mathrm{exp}}}}
\newcommand{\Genv}{{G_{\mathrm{env}}}}
\newcommand{\Lguide}{{\mathcal{L}_{\mathrm{G}}}}
\newcommand{\Lexec}{{\mathcal{L}_{\mathrm{E}}}}
\newcommand{\framework}{\textsc{Mage}}
\newcommand{\evokg}{\textsc{EvoKG}}
\title{MAGE: Multi-Agent Self-Evolution with Co-Evolutionary Knowledge Graphs}

\author{
 \textbf{Ruiyi Yang\textsuperscript{1,2}},
 \textbf{Zechen Li\textsuperscript{1}}, \\
 \textbf{Hao Xue\textsuperscript{3,2,1}},
 \textbf{Imran Razzak\textsuperscript{4,1}},
 \textbf{Flora D. Salim\textsuperscript{1,2}},
\\
 \textsuperscript{1}University of New South Wales, \\
 \textsuperscript{2}DataCruiser, \\
 \textsuperscript{3}The Hong Kong University of Science and Technology
(Guangzhou), \\
\textsuperscript{4}Mohamed Bin Zayed University of Artificial Intelligence,\\
 \small{
   \textbf{Correspondence:} \href{mailto:ruiyi.yang@unsw.edu.au}{ruiyi.yang@unsw.edu.au}
 }
}

\begin{document}

\maketitle
\input{sections/abstract}
\input{sections/introduction}
\input{sections/related_work}
\input{sections/method}
\input{sections/experiments}
\input{sections/conclusion}

% ── acknowledgments (hidden during review) ──
% \begin{ack}
% \end{ack}

\clearpage
\input{sections/acknowledgements}

\bibliographystyle{plain}
\bibliography{references}

%%%%%%%%%%%%%%%%%%%%%%%%%%%%%%%%%%%%%%%%%%%%%%%%%%%%%%%%%%%%
%% Appendix (technical, does not count toward page limit)
\newpage
\appendix
\input{sections/appendix}
\input{sections/rebuttal_additions}

%%%%%%%%%%%%%%%%%%%%%%%%%%%%%%%%%%%%%%%%%%%%%%%%%%%%%%%%%%%%
%% Force all remaining floats out, then start checklist on new page
\clearpage

%% NeurIPS Paper Checklist (mandatory, after appendix, not content pages)
{
\def\thesection{}
\input{checklist_camera_ready}
}

\end{document}

%% file: sections/abstract.tex
\begin{abstract}
Self-evolving language-model agents must decide what to learn next and how to preserve what they have learned across iterations. Existing systems typically carry this cross-iteration knowledge as natural-language feedback, flat episodic memory, or implicit reinforcement signals, none of which cleanly supports a frozen weak backbone at inference time. This paper introduces \framework{} (\textbf{M}ulti-\textbf{A}gent \textbf{G}raph-guided \textbf{E}volution), a framework that externalizes self-knowledge into a four-subgraph co-evolutionary knowledge graph. Its experience subgraph stores both teacher-written failure corrections and the learner's own past correct reasoning traces, which are retrieved as task-conditioned guidance for a frozen execution model. During evolution, the graph, a task-level search bandit, and a skill-level routing bandit are updated from the same reward stream, while the learner's backbone remains unchanged. We further provide structural analysis showing how append-only memory growth, bounded curriculum coverage, and task-filtered retrieval together support stable improvement of the retrieval substrate for frozen-learner evolution. Across nine benchmarks spanning mathematical reasoning, multi-hop and open-domain question answering, spatio-temporal analysis, financial numerical reasoning, medical multiple-choice, an open-world survival game, and web navigation, \framework{} achieves strong performance against prompt-based frozen-backbone baselines. Ablations show that self-harvested success traces and teacher-written corrections are complementary, with success memories contributing most on reasoning-template-heavy tasks and corrective memories supporting harder composition and interaction settings.  Code is available in https://github.com/RuiyiYang01/mage-alg
\end{abstract}

%% file: sections/introduction.tex
\section{Introduction}
\label{sec:intro}

Self-evolving language-model agents aim to improve from their own experience after deployment. 
Most recent systems instantiate this idea with a two-actor loop: one component explores or proposes tasks, while another component is updated from the resulting trajectories. 
The update mechanism itself may vary, such as supervised fine-tuning, reinforcement learning, self-play, or retrieval-time adaptation. However, these systems face a common upstream question: \emph{what experience should be retained across iterations, and in what form should it be represented?}
Existing self-evolving agents typically answer this question by storing cross-iteration knowledge in one of three forms. 
Natural-language feedback~\citep{yangtoward,shinn2023reflexion} is easy to generate but can collapse into generic advice when the teacher is weak or the failure mode is systematic. 
Flat episodic memory~\citep{wu2025evolver,zhang2026memrl,xu2025mem,allard2026experiential} stores trajectories or principles, but provides little structure for curriculum selection or dependency-aware reuse. 
Implicit reinforcement signals~\citep{wang2025ragen,chen2025multi,chen2025mars,li2025agentic} can drive parameter updates, but make the learned curriculum difficult to inspect and require a malleable student model. 
Motivated by these limitations, this paper asks:
\emph{Can the carrier of self-evolution be moved from model parameters or free-form feedback into a structured graph, so that a frozen execution backbone can improve across iterations?}

\begin{figure}[!htbp]
\centering
\includegraphics[width=0.96\textwidth]{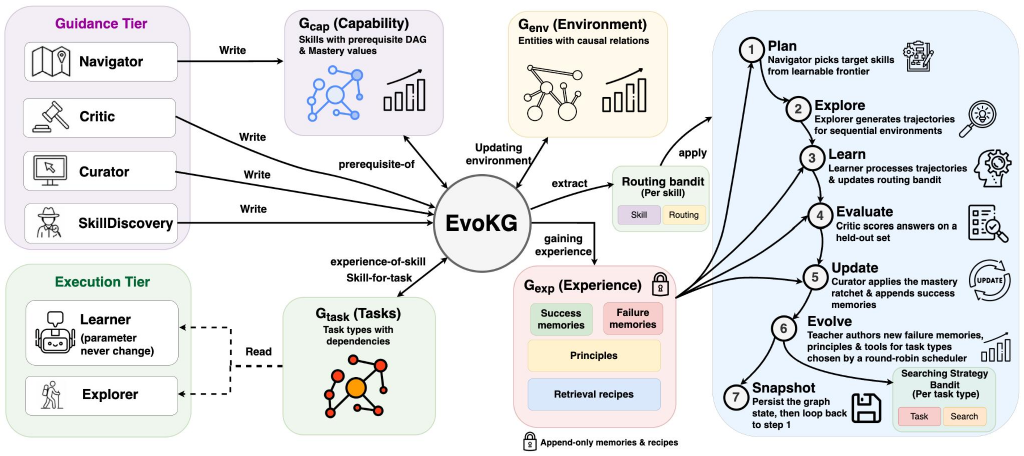}
\caption{
The \framework{} framework.
A strong \emph{guidance tier} ($\Lguide$) writes to the four-subgraph co-evolutionary knowledge graph (\evokg{}), while a frozen \emph{execution tier} ($\Lexec$) answers questions through a Learner that consults a dual success/failure memory index, a task-conditioned search-strategy bandit, and a per-skill routing bandit.
Correct answers are harvested back into the graph as success memories.
The graph and the two bandits carry the cross-iteration learning signal; the execution backbone remains frozen.
}
\label{fig:framework}
\vspace{-1em}
\end{figure}

We introduce \framework{} (\textbf{M}ulti-\textbf{A}gent \textbf{G}raph-guided \textbf{E}volution), a framework that externalizes self-knowledge into a four-subgraph co-evolutionary knowledge graph, \evokg{}. 
The graph contains capability, task, experience, and environment subgraphs. 
Its experience subgraph stores two complementary memory types: \emph{failure memories}, written by a strong guidance tier from evaluation errors, and \emph{success memories}, harvested from the frozen execution tier's own correct reasoning traces. 
At inference time, the learner retrieves task-conditioned memories and uses them as guidance for the frozen backbone. 
Across iterations, three state variables co-evolve under the same per-question reward stream: (i) the graph accumulates structured experience, (ii) a task-level search bandit learns retrieval strategies, and (iii) a skill-level routing bandit selects inference-time strategies. 
Thus, correct answers become future demonstrations, failures become future corrections, and the policies that produced them are updated in the same loop.

Our analysis formalizes the structural conditions under which this external memory substrate can support self-evolution in the frozen-backbone setting. 
Append-only growth of protected experience nodes yields information-monotonic graph growth; a recency-weighted round-robin curriculum gives a bounded coverage gap over observed task types; and an asymmetric mastery update bounds per-step forgetting. 
In conjunction with task-filtered retrieval, these properties provide a structural justification for using \evokg{} as the carrier of cross-iteration learning, while the empirical results test how this design behaves with real language models.

\paragraph{Contributions.}
\begin{enumerate}[leftmargin=*,itemsep=1pt]
    \item \textbf{Co-evolutionary graph memory for self-evolving agents.}
    We propose \evokg{}, a four-subgraph knowledge graph that represents capabilities, tasks, experience, and environment context, and preserves append-only success, failure, and principle memories as the external carrier of learning.

    \item \textbf{Frozen-backbone co-evolution through graph and bandit updates.}
    We couple the graph with a task-level search bandit and a skill-level routing bandit, all updated from the same per-question reward stream while the execution backbone remains unchanged.

    \item \textbf{Structural analysis of memory growth, coverage, and stability.}
    We prove information-monotonicity of append-only graph growth, a bounded coverage gap for the round-robin curriculum, and a stability bound for the asymmetric mastery ratchet; under task-filtered retrieval, these results characterize when the retrieval substrate can improve without parameter updates.

    \item \textbf{Evaluation across heterogeneous agent tasks.}
    We evaluate \framework{} on nine benchmarks spanning mathematical reasoning, multi-hop and open-domain question answering, spatio-temporal analysis, financial numerical reasoning, medical multiple-choice reasoning, open-world survival, and web navigation.
    The results show strong performance against frozen-backbone prompting baselines, and ablations show that self-harvested success memories provide the largest gains on reasoning-template-heavy tasks while teacher-written corrections remain complementary in other settings.
\end{enumerate}

%% file: sections/related_work.tex
\section{Related Work}
\label{sec:related}

\paragraph{Self-evolving language-model agents.}
Self-evolving agents aim to improve through iterative interaction, evaluation, and update.
\textsc{Exif}~\citep{yangtoward} instantiates this loop with a teacher-student design. A teacher explores an environment, generates instruction-trajectory pairs, supervises a student through SFT, and produces natural-language feedback for subsequent rounds.
Recent systems such as \textsc{Agent0}~\citep{xia2025agent0}, \textsc{AgentEvolver}~\citep{zhai2025agentevolver}, \textsc{Sage}~\citep{peng2026sage}, \textsc{Mae}~\citep{chen2025multi}, \textsc{Sirius}~\citep{zhao2025sirius}, \textsc{Absolute Zero}~\citep{zhao2025absolute}, and \textsc{InfiAgent}~\citep{yu2025infiagent} further explore curriculum generation, self-play, proposer-solver co-evolution, and reinforcement learning for agent improvement.
These methods demonstrate that iterative experience can improve agents, but the accumulated knowledge is usually carried either by parameter updates or by verbal feedback rather than by an explicit structured state.

\paragraph{Knowledge graphs and language-model agents.}
Knowledge graphs have been used to support language-model agents in retrieval, reasoning, and domain organization.
GraphRAG~\citep{edge2024local} and related graph-based retrieval methods organize external evidence for question answering, while \textsc{AgentiGraph}~\citep{zhao2025agentigraph} builds graph structures for domain-specific interactive agents.
\textsc{Agentic-KGR}~\citep{li2025agentic} studies multi-agent reinforcement learning for constructing knowledge graphs, and \textsc{Voyager}~\citep{wang2023voyager} maintains a growing skill library during open-ended embodied exploration.
These works show the utility of graph-structured or library-like external state, although the graph is typically used as a retrieval artifact, a constructed output, or a skill repository rather than as the evolving substrate that organizes cross-iteration learning.

\paragraph{Experience-driven memory.}
A growing line of work equips agents with non-parametric memory.
\textsc{EvolveR}~\citep{wu2025evolver} distills trajectories into abstract principles, \textsc{MemRL}~\citep{zhang2026memrl} learns over an episodic-memory store, and \textsc{A-Mem}~\citep{xu2025mem}, \textsc{Erl}~\citep{allard2026experiential},\textsc{SPLIT-RAG}~\citep{yang2025divide}, and \textsc{ReMem}~\citep{wei2025evo} explore dynamically indexed memories over experiences, principles, or triplets.
These approaches externalize part of the learning process from model parameters and make past experience available at inference time.
However, most memory stores are organized as flat collections whose entries are retrieved by similarity, with limited structure for representing capability dependencies, task coverage, or the different roles of successful and failed trajectories.

\paragraph{Curriculum, skill discovery, and dynamic prompting.}
Curriculum and skill discovery are central to open-ended agents.
\textsc{Voyager}~\citep{wang2023voyager} grows a code-based skill library, \textsc{SeAgent}~\citep{sun2025seagent} learns software use through an auto-generated curriculum, and proposer--solver systems such as \textsc{Agent0} and \textsc{Absolute Zero} construct tasks through self-play or adversarial generation.
In parallel, prompting methods such as self-consistency~\citep{wang2022self}, ReAct~\citep{yao2022react}, and Reflexion~\citep{shinn2023reflexion} improve inference-time behavior through sampling, tool use, or verbal self-reflection.
These lines highlight the importance of deciding what experience to expose to the model, but their prompts, memories, or curricula are usually fixed, locally retrieved, or generated without an explicit coverage guarantee over task types.

\paragraph{Skill evolution and efficient coordination.}
MemSkill~\citep{zhang2026memskill} learns and evolves memory-management skills, while SkillRL~\citep{xia2026skillrl} co-evolves a skill library with a parameter-updated policy.
SquidAgent~\citep{lin2026squidagent} uses predicted output-token costs to decide when to parallelize, combining inherited worker context with shared conventions to reduce coordination overhead.
Its focus on within-task execution efficiency complements \framework{}'s cross-iteration evolution of persistent experience; neither contribution implies that adding agents or inference calls is uniformly beneficial.
Appendix~\ref{app:skill_literature} compares additional contemporaneous skill-evolution methods.

\paragraph{Position of \framework{}.}
\framework{} combines these threads by treating the agent's persistent state as a co-evolutionary knowledge graph rather than as free-form feedback, a flat memory bank, or an implicit reward signal.
The graph links capabilities, task types, experience, and environment state; separates teacher-written failure memories from self-harvested success traces; and co-evolves with a search bandit and a routing bandit under the same reward stream.
This design makes the graph the carrier of cross-iteration learning while keeping the execution backbone frozen, and uses graph structure to couple what is learned, what is retrieved, and which task types are revisited.

%% file: sections/method.tex
\section{Method}
\label{sec:method}

\framework{} keeps the inference-time backbone frozen and stores cross-iteration learning in an explicit, inspectable graph. 
The graph is a typed directed multigraph $\KG = (\Gcap, \Gtask, \Gexp, \Genv)$ whose protected experience nodes include guidance-written principles, failure memories from evaluation errors, and success memories harvested from correct learner answers.
Each protected node is tagged with the task type it resolves and the skill it exercises, enabling task- and capability-conditioned retrieval.
At inference time, the Learner does not call the guidance tier or update model parameters; it conditions the frozen backbone on a task-filtered slice of $\Gexp$ retrieved by embedding similarity.
The remainder of this section defines the graph substrate (\S\ref{sec:method:evokg}), the two-tier architecture and dual memory index (\S\ref{sec:method:tier}), the evolution loop (\S\ref{sec:method:loop}), the curriculum and mastery updates (\S\ref{sec:method:curriculum}), and the inference-time retrieval rule (\S\ref{sec:method:cond}); full proofs and implementation details are deferred to Appendices~\ref{app:proofs}--\ref{app:agents}.

\subsection{The Co-Evolutionary Knowledge Graph}
\label{sec:method:evokg}

The co-evolutionary knowledge graph is a typed directed multigraph $\KG = (\Gcap, \Gtask, \Gexp, \Genv)$.
The capability subgraph $\Gcap$ stores skills with mastery values in $[0,1]$ and prerequisite edges.
The task subgraph $\Gtask$ stores task types with dependency edges and a $\textsc{skill\_for\_task}: \Gcap \to \Gtask$ resolver.
The experience subgraph $\Gexp$ stores five classes of experience node: \textbf{principles}, \textbf{failure memories} authored by the guidance tier, \textbf{success memories} harvested from the frozen learner's own correct reasoning, retrieval recipes, and abstracted patterns.
The environment subgraph $\Genv$ stores observed environment entities, relations, and task context used by the agent.

% Moved verbatim from Appendix B for a concrete main-text graph definition.
\paragraph{Capability subgraph $\Gcap$.}
Nodes are skills with attributes (name, mastery $\in [0,1]$, prompt template, strategy, principles).
Edges are $\textsc{prerequisite\_of}$ and $\textsc{composes\_into}$; the curator preserves acyclicity so the Navigator's topological sort is well defined.

The curator preserves an append-only invariant on the three protected experience classes: principle, failure-memory, and success-memory nodes are never removed once committed, even when mutable working slots such as mastery values, prompt templates, or strategies are revised.
This invariant gives the graph a monotone protected memory core, which later supports the retrieval-side analysis.

\begin{theorem}[\evokg{} Information Monotonicity, Appendix~\ref{app:thm_monotone}]
\label{thm:monotone}
Under the append-only invariant on principle, failure-memory, and success-memory nodes, $I(Y; \KG_k) \leq I(Y; \KG_{k+1})$ for all $k$, where $Y$ is the answer random variable on a fixed task distribution and $\KG_k$ is the graph state at the end of iteration $k$.
\end{theorem}

\subsection{Two-Tier Architecture and the Dual Memory Index}
\label{sec:method:tier}
\label{sec:method:dualmem}

\framework{} separates graph writing from inference-time execution.
The \emph{guidance tier} $\Lguide$ is a stronger model used only for graph writes, managing skill discovery, principle extraction, failure-memory authoring, dynamic-tool generation, retrieval-recipe authoring, and Navigator refinement.
The \emph{execution tier} $\Lexec$ is an open-weight frozen model used for trajectory generation and inference-time question answering.
Thus, $\Lguide$ is absent from the inference path, and $\Lexec$ never receives parameter updates.

The experience subgraph is indexed by two parallel embedding stores with shared retrieval logic but distinct payloads.
The \textbf{failure index} stores teacher-written worked examples that provide corrective reasoning for observed errors.
The \textbf{success index} stores the learner's own correct reasoning traces, and for multi-step questions, an explicit $(\text{skill}, \text{step output})$ decomposition that harvested from correctly solved evolution examples.
These traces allow iteration $N$'s correct answers to become iteration $N{+}1$'s demonstrations.

% Moved verbatim from the success-memory appendix.
\paragraph{Success-memory harvest structure.}
\label{sec:method:success_mem}

A success-memory node records a single $t$-tagged correct answer and four payload fields: the verbatim question, the frozen learner's raw chain-of-thought up to a character cap, the correct answer string, and for multi-step queries an explicit decomposition into $(\text{skill\_name}, \text{step\_output})$ tuples that preserves the intermediate reasoning steps the learner took to reach the answer.
The decomposition field is the channel that lets a future similar multi-step question retrieve not only ``here is a question-answer pair'' but ``here is the full reasoning trace and the sub-step pattern that worked,'' which is what distinguishes the success-memory channel from ordinary few-shot exemplars.

Finally, two lightweight bandits adapt how the graph is used: a per-skill \textbf{routing bandit} over inference strategies, and a per-task-type \textbf{search bandit} over retrieval strategies.
Both update from the same per-question correctness reward that drives graph writes, so memory growth, search adaptation, and routing adaptation co-evolve while $\Lexec$ remains frozen.

\subsection{Evolution Loop}
\label{sec:method:loop}

\paragraph{Core loop and scope.}
Retrieve task-conditioned experience, execute graph-derived prerequisites, append evaluated successes and corrections, and roll back harmful mutable-state updates.
The search and routing policies are optional, activation-dependent extensions; task search has no significant measured gain.
Environment state, principles, recipes, and patterns are task-dependent support, without a claim of independent benefit on every task (Appendix~\ref{app:controlled_components}).

Six role-specific agents cooperate through the loop in Algorithm~\ref{alg:loop}: SkillDiscovery, Navigator, Explorer, Learner, Critic, and Curator.
Their tier assignments and roles are listed in Appendix~\ref{app:agents}.
Figure~\ref{fig:iteration_trace} illustrates one iteration of memory writes, bandit updates, and next-iteration retrieval.
A delta-guard protects the mutable working state of the graph.
If the post-update accuracy drops by more than $\delta$, \framework{} rolls back mutable slots such as mastery values, prompt templates, and strategies.
Protected memory and principle nodes appended during the iteration are retained, preserving the append-only invariant.
The EVOLVE step rotates through four action types: principle extraction, dynamic-tool authoring, prompt refinement, and skill splitting, thus graph growth is not restricted to a single update mode.

\begin{figure}[t]
\centering
\includegraphics[width=0.95\textwidth]{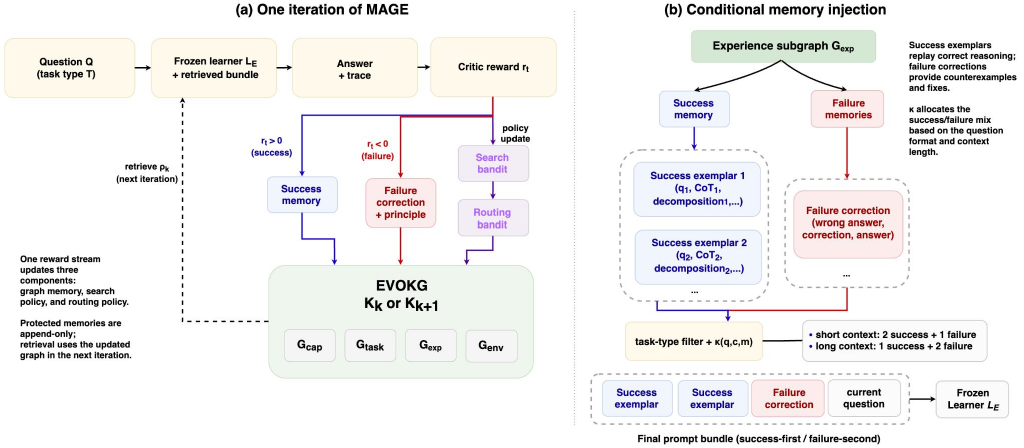}
\caption{
One-iteration co-evolution and conditional memory injection in \framework{}.
The frozen learner answers using a task-conditioned bundle of success exemplars and failure corrections retrieved from \evokg{}.
The critic produces reward $r_t$, which appends correct answers as success memories, triggers guidance-written corrections for errors, and updates the search and routing bandits.
Thus memory growth, retrieval, and policy adaptation are coupled through one reward stream.
}
\label{fig:iteration_trace}
\end{figure}

\begin{algorithm}[t]
\caption{\framework{} evolution loop (one iteration).}
\label{alg:loop}
\small
\begin{algorithmic}[1]
\STATE $\mathcal{P}_k \gets \textsc{Navigator}(\KG_{k-1})$ \hfill \COMMENT{PLAN}
\STATE $\mathcal{D}_k \gets \textsc{Explorer}(\mathcal{P}_k, \mathcal{E}, \Lexec)$ \hfill \COMMENT{EXPLORE; sequential environments only}
\STATE $E_k, \mathcal{S}_k \gets \textsc{Critic}\bigl(\textsc{Learner}(\KG_{k-1}), \mathcal{E}\bigr)$ \hfill \COMMENT{EVALUATE; harvest success memories}
\STATE $\KG_k' \gets \textsc{Curator}.\textnormal{update}(\KG_{k-1}, E_k, \mathcal{S}_k)$ \hfill \COMMENT{UPDATE; ratchet and append $\mathcal{S}_k$}
\STATE $\KG_k \gets \textsc{Evolve}(\KG_k', E_k, k, \Lguide)$ \hfill \COMMENT{EVOLVE; write failures, principles, tools}
\STATE \textsc{Snapshot}$(\KG_k)$; roll back mutable slots if $\textnormal{acc}(\KG_k) < \textnormal{acc}(\KG_{k-1}) - \delta$
\end{algorithmic}
\end{algorithm}

\subsection{Curriculum and Mastery Updates}
\label{sec:method:curriculum}
\label{sec:method:robin}
\label{sec:method:ratchet}

The Navigator first computes the \emph{learnable frontier} of $\Gcap$: skills whose prerequisites are mastered but whose own mastery remains below threshold.
The EVOLVE phase then selects which task types receive new failure memories and principles.
A purely failure-count-based selector can repeatedly focus on the same high-error task types, so \framework{} uses a recency-weighted round-robin score,
\begin{equation}
\label{eq:roundrobin}
s(t) \;=\; n_{\text{fail}}(t) + \lambda \bigl(k - k_{\text{last}}(t)\bigr), \qquad \lambda > 0,
\end{equation}
and selects the top-$M$ scoring task types, where $n_{\text{fail}}(t)$ is the failure count for task type $t$ and $k_{\text{last}}(t)$ is the most recent iteration in which $t$ was selected.

The Curator updates per-skill mastery using an asymmetric exponential moving average that increases quickly after successful evidence but decays slowly after poor measurements:
\begin{equation}
\label{eq:ratchet}
m_{k} =
\begin{cases}
\alpha\, e_k + (1 - \alpha)\, m_{k-1}, & e_k \geq m_{k-1}, \\
m_{k-1} - \gamma\,(m_{k-1} - e_k),     & e_k <    m_{k-1},
\end{cases}
\qquad \alpha > \gamma \in (0,1),
\end{equation}
where $e_k$ is the Critic's measured success rate.
The selector and mastery update give two structural properties.

\begin{theorem}[Bounded Coverage Gap, Appendix~\ref{app:thm_coverage}]
\label{thm:coverage}
Under Eq.~\eqref{eq:roundrobin} with at most $M$ targets per iteration, every observed task type $t$ is reselected within $\lceil n_{\max}/\lambda + N/M \rceil$ iterations, where $N$ is the number of observed task types and $n_{\max} = \max_t n_{\text{fail}}(t)$.
\end{theorem}

\begin{theorem}[Mastery Ratchet Stability, Appendix~\ref{app:thm_ratchet}]
\label{thm:ratchet}
Under Eq.~\eqref{eq:ratchet}, $m_k \geq (1-\gamma)^j \max_{k' \leq k-j} m_{k'}$ for any window $j$: a structural lower bound on running mastery in terms of any past peak.
\end{theorem}

Together, Theorem~\ref{thm:coverage} bounds how long an observed task type can be left unselected, while Theorem~\ref{thm:ratchet} bounds per-step mastery decay.
Both properties follow from the update rules themselves rather than from implementation-specific heuristics.

\subsection{Conditional Memory Injection at Inference}
\label{sec:method:cond}

At inference time, given a question $q$ of task type $t$ with context $c$, the Learner retrieves the top-$K$ task-type-filtered memories from both the success and failure indices.
A format-conditional allocation $\kappa$ determines the mixture: short-context questions receive a success-heavy bundle, while long-context questions with rich gold context receive a failure-heavy bundle.
The retrieved memories are formatted success-first and failure-second as in-context exemplars before being prepended to the question.
Formally,
\begin{equation}
\label{eq:rho}
\rho_k(\cdot \mid q, t, c)
\;=\;
\textsc{TopK}\Bigl(
\textsc{embed}(q),\;
\bigl\{\,m \in \Gexp_k \;:\; m.\text{task\_type} = t \wedge \kappa(q, c, m)\,\bigr\}
\Bigr),
\end{equation}
where $\kappa$ enforces the format-conditional success/failure allocation.
The Learner then emits
\begin{equation}
\label{eq:learner_emit}
\hat{y}
=
\Lexec\bigl(
\textsc{format}(\rho_k(\cdot)) \,\Vert\, \textsc{prompt}(q, c, \KG_k)
\bigr).
\end{equation}

We state the final result under an idealized retrieval-error model: additional task-matched memories expand the oracle top-$K$ candidate set, and $\varepsilon_K$ bounds the total-variation gap between the embedding retriever and this oracle selector.
Together, task-type filtering, append-only graph growth, and bounded curriculum coverage yield a retrieval-side support result.

\begin{theorem}[Task-Filtered Retrieval Support, Appendix~\ref{app:thm_frozen}]
\label{thm:frozen}
Let $A_k(t) = \Expect_{q \sim \mathcal{D}_t,\, \mu \sim \rho_k}\bigl[\textnormal{acc}(\Lexec(\textnormal{format}(\mu)\,\Vert\, q, c), y)\bigr]$.
Under Theorems~\ref{thm:monotone}--\ref{thm:coverage} and the task-type filter, $A_{k+1}(t) \geq A_k(t) - \varepsilon_K$ for every observed $t$, where $\varepsilon_K$ is the rank-$K$ retrieval error of the embedding index.
\end{theorem}

Theorem~\ref{thm:frozen} is a structural \emph{retrieval-side} result, not an unconditional guarantee about arbitrary LLM behavior.
Append-only graph growth and bounded coverage prevent the task-conditioned evidence set from shrinking, while $\varepsilon_K$ measures the gap between embedding retrieval and an oracle top-$K$ selector.
Whether the frozen backbone can exploit this enriched retrieval substrate is evaluated empirically through ablations and cross-teacher analysis.

%% file: sections/experiments.tex
\section{Experiments}
\label{sec:exp}

\subsection{Setup}
\label{sec:exp:setup}

We evaluate \framework{} on nine benchmarks drawn from six task families.
\textbf{GSM8K}~\citep{cobbe2021training} and \textbf{RealMath}~\citep{zhang2025realmath} test mathematical reasoning.
\textbf{HotpotQA}~\citep{yang2018hotpotqa} and \textbf{WebQA}~\citep{berant2013semantic} test multi-hop and open-domain factoid QA.
\textbf{STBench}~\citep{li2025stbench} tests spatio-temporal analysis across twenty-seven task types.
\textbf{FinQA}~\citep{chen2021finqa} tests table-grounded numerical reasoning over financial reports, and \textbf{MedQA-USMLE}~\citep{jin2021disease} tests medical multiple-choice reasoning.
\textbf{Crafter}~\citep{hafner2021benchmarking} is an open-world sequential survival game, and \textbf{WebShop}~\citep{yao2022webshop} is a web-navigation task with parameterised actions.
FinQA and MedQA-USMLE are used only as standardized reasoning benchmarks; \framework{} is not evaluated or proposed as a financial-advice or clinical decision-support system.

\paragraph{Evaluation protocol and leakage control.}
For each benchmark, we separate the examples used for graph evolution from the final held-out evaluation pool used for reporting results.
Success memories, failure memories, principles, and bandit updates are produced only from the evolution pool.
After evolution, \evokg{} and the two bandits are frozen, and all reported test metrics are computed without writing any test example back into the graph.
For benchmarks with native exact-match metrics, including FinQA and MedQA-USMLE, we report automatic scores directly.
For free-form QA and reasoning benchmarks, we use the semantic judge protocol described below and provide automatic-metric cross-checks where available.

\paragraph{Models.}
We instantiate \framework{} along two orthogonal axes: a frozen execution tier using one of two open-weight 8B instruction-tuned models (Qwen3-8B and Llama-3.1-8B-Instruct) served via vLLM at temperature $0$, and a guidance tier using one of three closed-weight teachers (Claude Opus 4.6, Sonnet 4.6, and Haiku 4.5) used only during graph evolution.
The main reasoning table reports Qwen3-8B with Sonnet and Opus guidance, while the domain-extension table reports Llama-3.1-8B with the same guidance tiers; the full teacher--learner matrix and additional multi-seed results are reported in Appendices~\ref{app:teacher_learner} and~\ref{app:variance}.
All hyperparameters are fixed across benchmarks (Appendix~\ref{app:hyperparams}).

\paragraph{Baselines and scoring.}
On the reasoning benchmarks, we compare against frozen-backbone prompt-based scaffolds: zero-shot chain-of-thought, eight-shot CoT, self-consistency with ten samples~\citep{wang2022self}, ReAct~\citep{yao2022react}, and Reflexion~\citep{shinn2023reflexion}.
These baselines isolate the frozen-backbone setting: all methods use the same execution backbone at inference time and differ only in the external scaffold placed around it.
For sequential environments, we compare against same-backbone standalone baselines and published WebShop references under the same catalog setting.
For judge-scored benchmarks, all methods use the same answer-extraction pipeline and the same Claude Sonnet~4.6 semantic judge on the evaluation pool; judge prompts and automatic cross-checks are provided in Appendix~\ref{app:hotpotqa_emf1}.
Published trained or domain-specific systems are included only as references, not as strict frozen-backbone baselines.

\subsection{Main Results}
\label{sec:exp:reasoning}

\begin{table}[!htbp]
\vspace{-6pt}
\centering
\caption{\framework{} versus prompt-based scaffolds on the five reasoning benchmarks. All methods use Qwen3-8B as the frozen backbone and are scored by the same Claude Sonnet~4.6 judge on 200 held-out questions per benchmark. The \framework{} columns report mean$\pm$std over multiple seeds where available; cell-level $n$ is shown in subscript. $\Delta$ is the gap of the best \framework{} cell to the strongest frozen-backbone baseline.}
\label{tab:reasoning}
\small
\resizebox{\textwidth}{!}{
\begin{tabular}{lrrrrr|rrr}
\toprule
& \multicolumn{5}{c}{Prompt-based baselines (Qwen3-8B, frozen)} & \multicolumn{2}{c}{\framework{} (Qwen3-8B)} & \\
\cmidrule(lr){2-6}\cmidrule(lr){7-8}
Benchmark & 0-shot CoT & 8-shot & SC$_{10}$ & ReAct & Reflexion & Sonnet & Opus & $\Delta$ \\
\midrule
GSM8K    & $64.5$ & $82.5$ & $72.0$ & $80.1$ & $63.0$ & $88.2 \pm 2.9_{n=5}$ & $\mathbf{90.4 \pm 3.8}_{n=5}$ & $+7.9$ \\
HotpotQA & $84.5$ & $78.5$ & $80.5$ & $70.5$ & $78.5$ & $88.5 \pm 2.0_{n=5}$ & $\mathbf{91.5 \pm 1.8}_{n=5}$ & $+7.0$  \\
WebQA    & $55.0$ & $42.0$ & $46.5$ & $35.5$ & $34.5$ & $59.6 \pm 3.4_{n=5}$ & $\mathbf{62.7 \pm 1.8}_{n=5}$ & $+7.7$  \\
STBench  & $35.5$ & $12.0$ & $20.0$ & $31.8$ & $37.0$ & $54.8 \pm 2.1_{n=5}$ & $\mathbf{57.8 \pm 2.9}_{n=5}$ & $+20.8$  \\
RealMath & $53.5$ & $32.5$ & $33.5$ & $43.5$ & $34.5$ & $\mathbf{85.1 \pm 6.7}_{n=5}$ & $82.8 \pm 3.5_{n=5}$ & $+31.6$ \\
\bottomrule
\end{tabular}
}
\end{table}

\begin{wrapfigure}{R}{0.48\textwidth}
\vspace{-0.6em}
\centering
\includegraphics[width=0.48\textwidth]{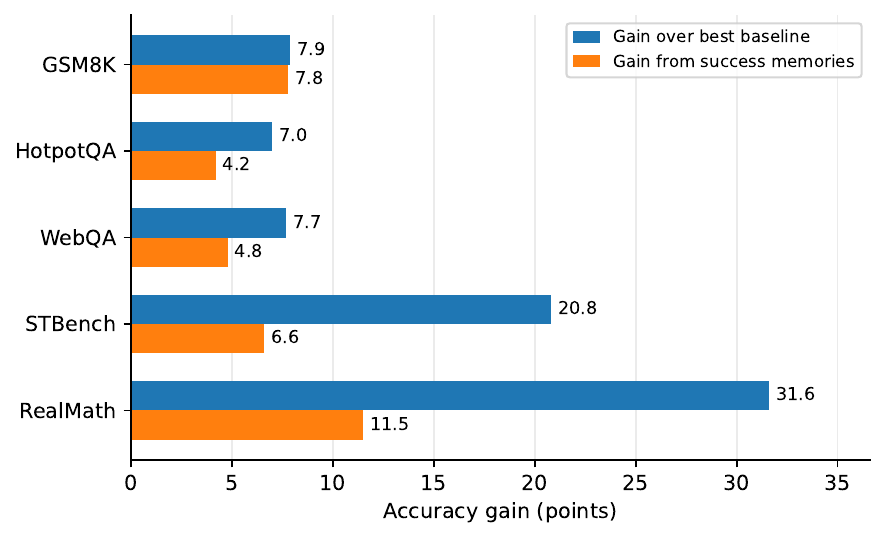}
\caption{Effect sizes of \framework{} on the reasoning benchmarks: gain over the strongest frozen-backbone baseline and gain attributable to success memories.}
\label{fig:gains}
\vspace{-0.8em}
\end{wrapfigure}
Tables~\ref{tab:reasoning}--\ref{tab:sequential_main} summarize the main results across the original reasoning suite, the finance/medical domain extensions, and the sequential environments.
On the core reasoning suite, \framework{} outperforms the strongest frozen-backbone prompting baseline on all five benchmarks.
The largest gains appear on tasks that require reusable reasoning templates or task-type-specific retrieval, such as RealMath and STBench, while the gains are smaller but still positive on benchmarks where prompt-only baselines are already strong, such as GSM8K and HotpotQA.
The domain extensions test whether the same mechanism transfers beyond the original reasoning suite.
On FinQA and MedQA-USMLE, \framework{} improves over the standalone Llama-3.1-8B baseline using benchmark-native exact-match scoring rather than an LLM judge.
On FinQA, \framework{} also exceeds the published frozen-8B financial reasoning reference; on MedQA, it narrows the gap to stronger prompt-engineered medical references without task-specific rationale pre-generation.
These results suggest that graph-carried memory can support domain-specific numerical and multiple-choice reasoning, although the published domain-specific references are included only as contextual references rather than strict same-scaffold baselines.

\begin{table}[!htbp]
\centering
\caption{Domain extension to math, finance and medicine. FinQA uses numeric exact-match, and MedQA-USMLE uses letter exact-match.}
\label{tab:reasoning_extension}
\small
\begin{tabular}{lrr|rr}
\toprule
& \multicolumn{2}{c}{Llama-3.1-8B baselines} & \multicolumn{2}{c}{\framework{} (Llama-3.1-8B)} \\
\cmidrule(lr){2-3}\cmidrule(lr){4-5}
Benchmark & 0-shot CoT & Frozen-8B reference & Sonnet & Opus \\
\midrule
GSM8K    & $64.5$ & $84.5$ (8-shot, Meta) & $\mathbf{92.5} \pm 4.3_{n=5}$               & $92.0 \pm 3.2_{n=5}$ \\
FinQA    & $58.0$ & $60.87$ (Fino1, SFT)  & $67.5 \pm 1.7_{n=5}$     & $\mathbf{69.0 \pm 1.5_{n=5}}$\\
MedQA    & $64.5$ & $74$ (Medprompt) & $74.6 \pm 2.3_{n=5}$ & $\mathbf{79.8} \pm 1.9_{n=5}$ \\
\bottomrule
\end{tabular}
\vspace{-6pt}
\end{table}

\begin{table}[t]
\centering
\caption{\framework{} on the two sequential environments. Crafter follows the BALROG protocol and reports peak achievement-unlock rate across iterations. WebShop follows the standard test protocol on the same hundred-product catalog as RetroAgent; RetroAgent is reproduced as published references.}
\label{tab:sequential_main}
\setlength{\tabcolsep}{5pt}
\small
\begin{tabular}{llrr}
\toprule
Environment & Configuration & Task score / peak (\%) & Strict success (\%) \\
\midrule
\multirow{4}{*}{Crafter}
 & Standalone Llama-3.1-8B (BALROG, 3 seeds)   & $25.5 \pm 3.2$  & --- \\
 & \framework{} Sonnet teacher ($n=5$)          & $33.5 \pm 2.3$  & --- \\
 & \framework{} Opus teacher ($n=5$)            & $\mathbf{37.9 \pm 3.5}$  & --- \\
 & \framework{} Haiku teacher ($n=5$)           & $34.8 \pm 1.7$  & --- \\
\midrule
\multirow{4}{*}{WebShop}
 & RetroAgent (in-context, Qwen-7B reference)   & $87.6$          & $78.9$ \\
 & RetroAgent (RL-trained reflection)           & $88.9$          & $82.3$ \\
 & \framework{} Sonnet teacher ($n=5$)          & $\mathbf{90.2 \pm 2.3}$ & $\mathbf{84.0 \pm 1.8}$ \\
 & \framework{} Opus teacher ($n=5$)            & $89.7 \pm 1.6$  & $82.0 \pm 1.0$ \\
\bottomrule
\end{tabular}
\vspace{-1em}
\end{table}

On sequential environments, \framework{} remains competitive.
WebShop shows strong performance relative to published references under the same catalog setting. On Crafter, all teacher-guided configurations improve over the standalone Llama-3.1-8B BALROG reference.
Several prompt-based scaffolds underperform zero-shot CoT on the frozen 8B backbones, suggesting that additional inference-time structure is not automatically beneficial at this scale.
In contrast, \framework{} improves by accumulating task-conditioned experience and adapting retrieval/routing policies across iterations. 
Automatic scoring cross-checks for benchmarks with native metrics are provided in Appendix~\ref{app:hotpotqa_emf1}, and the full teacher--learner matrix is reported in Appendix~\ref{app:teacher_learner}.
Figure~\ref{fig:gains} summarizes the main effect sizes behind Tables~\ref{tab:reasoning} and~\ref{tab:ablation_success}. 
\framework{} improves over the strongest frozen-backbone prompting baseline on all five reasoning benchmarks, and the success-memory ablation shows positive gains across the same benchmark family.

\subsection{Teacher Usage and Inference Cost}
\label{sec:exp:cost}

\begin{table}[!htbp]
\vspace{-10pt}
\centering
\small
\caption{Teacher usage rate across completed \framework{} training runs. Training-time fraction is the share of language-model calls routed to the closed-weight guidance tier ($\Lguide$); the remainder go to the local frozen execution tier ($\Lexec$). Inference-time fraction is measured after \evokg{} with frozen bandits.}
\label{tab:cost}
\resizebox{\textwidth}{!}{
\begin{tabular}{lrrrr}
\toprule
Phase / Benchmark & $n$ runs & Guidance calls (mean) & Execution calls (mean) & Guidance fraction \\
\midrule
\textbf{Aggregate (all training runs)} & $52$ & $\mathbf{7{,}122}$ (total) & $\mathbf{336{,}272}$ (total) & $\mathbf{2.12\%}$ \\
\textbf{Inference / test (all runs)}   & ---  & $\mathbf{0}$            & ---                          & $\mathbf{0.00\%}$ \\
\midrule
WebShop          & $5$  & $18$    & $1{,}699$  & $0.1$--$1.9\%$ \\
Crafter          & $11$ & $118$   & $22{,}508$ & $0.0$--$0.9\%$ \\
HotpotQA         & $4$  & $212$   & $5{,}362$  & $3.0$--$4.6\%$ \\
GSM8K            & $3$  & $229$   & $3{,}651$  & $4.1$--$7.5\%$ \\
RealMath         & $2$  & $355$   & $3{,}316$  & $7.0$--$14.0\%$ \\
STBench          & $2$  & $385$   & $2{,}948$  & $10.2$--$14.0\%$ \\
WebQA            & $3$  & $480$   & $3{,}134$  & $11.9$--$14.7\%$ \\
MedQA-USMLE      & $6$  & $167$   & $5{,}844$  & $2.0$--$3.3\%$ \\
FinQA            & $6$  & $899$   & $18{,}805$ & $4.2$--$4.9\%$ \\
\bottomrule
\end{tabular}
}
\end{table}

A central design goal of \framework{} is to separate training-time guidance from deployment-time inference: the stronger guidance tier writes to the graph during evolution, but the final inference loop runs on the frozen execution tier.
Table~\ref{tab:cost} counts guidance-tier and execution-tier calls across completed training runs. Across completed runs, only $\mathbf{1.91\%}$ of language-model invocations are routed to the guidance tier during graph evolution.
After graph construction, the inference-time guidance fraction is $\mathbf{0.00\%}$: all deployment-time queries are handled by the frozen execution model with retrieved graph memories.
This separates the cost of building \evokg{} from the cost of using it, and makes deployment independent of closed-weight guidance calls.

\paragraph{Inference latency.}
Per-question wall-clock latency on a single H100 80GB at fp16 ranges from $5.66$ to $28.33$ seconds for \framework{} after graph construction, with no guidance-tier calls at inference time, faster than self-consistency-$N{=}10$ on every measured benchmark and competitive with or faster than ReAct and Reflexion.
Detailed accuracy--latency comparisons are provided in Appendix~\ref{app:latency}.

\subsection{Success-Memory Ablation}
\label{sec:exp:ablation}

The success-memory ablation tests whether the learner's own correct reasoning traces provide a distinct source of cross-iteration learning beyond teacher-written failure memories.
In the \emph{No-success} variant, the harvest loop is disabled while failure memories, principles, and the remaining graph updates are kept unchanged. Table~\ref{tab:ablation_success} shows that self-harvested success memories improve performance on all five reasoning benchmarks.
The largest gains appear on mathematical reasoning, where correct traces provide reusable reasoning templates that can be replayed on future questions.
STBench also benefits substantially, suggesting that task-conditioned self-replay helps when the benchmark contains many recurring task types.
The gains on HotpotQA and WebQA are smaller but remain positive, indicating that success memories are still useful even when rich context or factoid answer formats reduce the need for template reuse.
These results support the dual-memory design: success traces and failure corrections are complementary rather than interchangeable.
Failure memories provide corrective guidance for observed mistakes, while success memories preserve reasoning patterns that the frozen learner has already executed correctly.
Additional sequential-environment ablations and graph-growth measurements are provided in Appendices~\ref{app:sequential_ablation} and~\ref{app:kg_growth}.

To make the memory mechanism concrete, Appendix~\ref{app:bundle_example} provides detailed retrieved bundles from GSM8K and HotpotQA to illustrate the mechanism behind Table~\ref{tab:ablation_success}. 
The GSM8K example combines learner-harvested success traces with teacher-written failure corrections, and its recovery trace shows previously failed questions being corrected after the memory bank grows. 
The HotpotQA example shows the same pattern in a multi-hop comparison setting: retrieved exemplars provide a reusable reasoning skeleton, while the correction memory specifies the attribute-comparison strategy. 
Together, these examples show that the dual-memory channel is not only an aggregate ablation effect, but also appears directly in the prompts consumed by the frozen learner.

\vspace{-0.6em}
\begin{table}[!htbp]
\centering
\small
\caption{Success-memory ablation across the reasoning benchmarks. \emph{No-success} disables harvest of the learner's own correct reasoning; teacher-written failure memories remain. Both columns report mean$\pm$std over matched seeds where available; $\Delta$ is the gap attributable to self-harvested experience.}
\label{tab:ablation_success}
\resizebox{\textwidth}{!}{
\begin{tabular}{lrrr}
\toprule
Benchmark & \framework{} (full) & No-success ablation & $\Delta$ (success memories) \\
\midrule
GSM8K    & $\mathbf{88.2 \pm 2.9}$ & $80.4 \pm 12.7$ & $+7.8$ \\
RealMath & $\mathbf{85.1 \pm 6.7}$ & $73.6 \pm 9.3$  & $+11.5$ \\
STBench  & $\mathbf{54.8 \pm 2.1}$ & $48.2 \pm 4.0$  & $+6.6$ \\
HotpotQA & $\mathbf{88.5 \pm 2.0}$ & $84.3 \pm 2.1$  & $+4.2$ \\
WebQA    & $\mathbf{59.6 \pm 3.4}$ & $54.8 \pm 1.9$  & $+4.8$ \\
\midrule
\emph{Mean across the five reasoning benchmarks} & & & $+7.0$ \\
\bottomrule
\end{tabular}
}
\vspace{-1.0em}
\end{table}

\subsection{Controlled Memory and Execution Comparisons}
\label{sec:exp:controlled}

Table~\ref{tab:controlled_main} adds matched memory controls and a prerequisite-removal intervention using the same frozen Qwen3-8B learner and evaluation instances within each comparison.
These use corrected paper-aligned EM/F1, with strict automatic scoring for MemSkill, separately from the semantic-judge tables above.

\begin{table}[!htbp]
\centering
\small
\caption{Matched comparisons (EM / F1, \%). A-MEM receives identical success/failure episodes; no-chain disables prerequisite traversal. RealMath pools three graphs and 600 paired runs; MemSkill uses HotpotQA-200 and a matched 1,024-token answer budget.}
\label{tab:controlled_main}
\begin{tabular}{llrr}
\toprule
Benchmark & Control & \framework{} & Control \\
\midrule
WebQA & A-MEM~\citep{xu2025mem} & $62.7 / 67.1$ & $62.7 / 69.4$ \\
STBench & A-MEM & $57.2 / 61.3$ & $51.7 / 52.0$ \\
RealMath & A-MEM & $82.4 / 84.5$ & $69.7 / 66.8$ \\
RealMath & No-chain & $82.4 / 84.5$ & $72.0 / 74.0$ \\
HotpotQA & MemSkill~\citep{zhang2026memskill} & $54.0 / 71.67$ & $40.5 / 55.34$ \\
\bottomrule
\end{tabular}
\end{table}

The traversal gain is positive in all three graphs, with an accuracy--compute trade-off: $2.52$ calls/$19.15$ seconds per query versus $1.00$/$12.23$ for no-chain.
A-MEM has higher WebQA F1, so structured execution is not uniformly preferable.
Appendix~\ref{app:controlled_baselines} reports the backbone-matched SkillRL comparison; Appendices~\ref{app:controlled_components}--\ref{app:graph_cases} provide component, local-teacher, transfer, growth, and case evidence.

%% file: sections/conclusion.tex
\section{Conclusion}
\label{sec:conclusion}

\framework{} studies self-evolution in the frozen-backbone setting, where cross-iteration learning must be carried by external state rather than model parameters. 
The central idea is to make this state explicit: a co-evolutionary knowledge graph stores capabilities, task types, environment context, teacher-written corrections, and the learner's own successful reasoning traces, while a search bandit and routing bandit adapt how this graph is used. 
The analysis shows that append-only protected memories, bounded task-type coverage, asymmetric mastery updates, and task-filtered retrieval provide a retrieval-side substrate for stable frozen-backbone self-evolution. 
Empirically, \framework{} is evaluated across nine benchmarks spanning mathematical reasoning, multi-hop and open-domain QA, spatio-temporal analysis, financial and medical reasoning, open-world survival, and web navigation. The results show that graph-carried experience can provide strong gains over frozen-backbone prompting baselines, with the largest improvements on tasks that benefit from reusable reasoning templates and task-conditioned retrieval.
Ablations further show that self-harvested success traces and teacher-written corrections are complementary: success memories are most useful when past correct reasoning can be replayed, while corrective memories and graph-derived principles support settings where error correction, composition, or interaction dominates.

\paragraph{Scope and limitation.}
\framework{} is most effective when task experience exposes recurring structure that can be organized into reusable memories, task types, and skill dependencies. Performance may therefore depend on the quality of the curator's graph writes. Future work should study broader teacher families and stronger cross-backbone transfer.

%% file: sections/acknowledgements.tex
\section*{Acknowledgments and Disclosures}
This research used the NVIDIA H100 cluster provided by SHARON AI.

%% file: sections/appendix.tex
\section{Proofs}
\label{app:proofs}

This appendix collects full proofs of the four theorems stated in the main text.
The structure of each proof closely tracks the proof sketch in Sec.~\ref{sec:method}; the appendix provides the precise statements, the auxiliary lemmas, and the corner cases that the main-text sketches elide for space.

\subsection{Proof of Theorem~\ref{thm:monotone} (\evokg{} Information Monotonicity)}
\label{app:thm_monotone}

\begin{proof}
Let $Y$ denote the answer random variable on a fixed evaluation distribution $\mathcal{D}$ and $\KG_k$ the \evokg{} state at the end of iteration $k$.
The curator's update procedure on $\Gexp$ is governed by two rules.
\textbf{(i)} Every principle, failure-memory, and success-memory node added during the EVALUATE and EVOLVE phases is appended to $\Gexp$ and is exempt from the low-confidence pruner (\texttt{prune\_low\_confidence} skips $\texttt{outcome} \in \{\texttt{principle}, \texttt{failure\_memory}, \texttt{success\_memory}\}$).
\textbf{(ii)} Pruning of other experience nodes (\texttt{outcome} $\in \{\texttt{success}, \texttt{failure}\}$, i.e.\ abstracted pattern summaries rather than worked-example memories) removes only nodes with confidence below a fixed threshold and is performed after the curator's append step.
Therefore, at the end of iteration $k+1$, the protected node sets satisfy
\[
\Gexp_{k+1}^{\text{prin}} \,\supseteq\, \Gexp_k^{\text{prin}}, \qquad
\Gexp_{k+1}^{\text{fail-mem}} \,\supseteq\, \Gexp_k^{\text{fail-mem}}, \qquad
\Gexp_{k+1}^{\text{succ-mem}} \,\supseteq\, \Gexp_k^{\text{succ-mem}},
\]
and the curator's working slots on $\Gcap$ (mastery, prompt template, strategy) and $\Gtask$ are updated in place but the corresponding nodes themselves are not deleted.
Hence the node-and-edge set of $\KG_{k+1}$ contains the node-and-edge set of $\KG_k$ at the level of principle, failure-memory, success-memory, capability, and task subgraphs; the only nodes that may be removed are the low-confidence abstracted-pattern nodes, which by construction carry confidence below the threshold and are independent of $Y$ given the rest of $\KG_k$.
Write $\KG_{k+1} = (\KG_k, \Delta_k)$ for the appended nodes $\Delta_k = \KG_{k+1} \setminus \KG_k$ (treating the deleted low-confidence nodes as conditionally independent of $Y$).
By the chain rule for mutual information,
\[
I(Y; \KG_{k+1}) \;=\; I(Y; \KG_k, \Delta_k) \;=\; I(Y; \KG_k) \;+\; I(Y; \Delta_k \mid \KG_k).
\]
Conditional mutual information is non-negative, so $I(Y; \KG_{k+1}) \geq I(Y; \KG_k)$, with equality iff $I(Y; \Delta_k \mid \KG_k) = 0$, i.e.\ iff $\Delta_k \perp Y \mid \KG_k$.

\paragraph{Failure mode under non-append-only pruning.}
The argument requires the append-only invariant on all three protected node classes.
A pruner that deletes worked-example memories whenever the underlying skill rises above some mastery threshold would violate it: the deletion step would remove nodes that carry information about $Y$ (the worked example for question $q$ does carry information about the answer to $q$), and the chain-rule argument would fail because $\KG_{k+1}$ would no longer be a refinement of $\KG_k$.
The curator's \texttt{prune\_low\_confidence} explicitly exempts all three protected node types for exactly this reason.
\end{proof}

\subsection{Proof of Theorem~\ref{thm:coverage} (Bounded Coverage Gap)}
\label{app:thm_coverage}

\begin{proof}
Let $\mathcal{T}$ denote the set of task types observed at least once in any iteration up to the current iteration $k$, $|\mathcal{T}| = N$.
Let $n_{\max} = \max_{t \in \mathcal{T}} n_{\text{fail}}(t)$ denote the maximum failure count over $\mathcal{T}$ at iteration $k$, and recall the selector
\[
s(t) \;=\; n_{\text{fail}}(t) + \lambda (k - k_{\text{last}}(t)),
\qquad \lambda > 0.
\]
At iteration $k$ the selector picks the top-$M$ task types by $s$.

Fix any task type $t^* \in \mathcal{T}$ and assume $t^*$ has not been selected for $g$ consecutive iterations starting from iteration $k_0 = k - g$.
Then $s_{k}(t^*) \geq n_{\text{fail}}(t^*) + \lambda g \geq \lambda g$.
For any other task type $t' \in \mathcal{T} \setminus \{t^*\}$, the score satisfies $s_{k}(t') \leq n_{\max} + \lambda \cdot (k - k_{\text{last}}(t'))$, which is at most $n_{\max} + \lambda g$ if $t'$ was last selected at the same iteration as $t^*$ was last selected, and is otherwise smaller (because more recent selection $\Rightarrow$ smaller recency bonus).
Hence whenever $\lambda g \geq n_{\max}$---equivalently, $g \geq n_{\max}/\lambda$---we have $s_k(t^*) \geq n_{\max}$, which is at least as large as any score that an incumbent in the top-$M$ achieved at the iteration in which it was most recently selected.

If $g \geq n_{\max}/\lambda$, then $t^*$ is therefore selected at iteration $k$ unless every one of the top-$M$ slots is occupied by a task type with strictly higher score than $t^*$.
A simple counting argument bounds the number of additional iterations during which the top-$M$ slots can all be occupied by task types other than $t^*$: each such iteration moves $M$ task types out of the ``not recently selected'' pool, so after at most $\lceil N/M \rceil$ further iterations the round-robin sweep through the pool exhausts all alternatives and $t^*$ enters the top-$M$.
Combining the two bounds, $t^*$ is selected within $\lceil n_{\max}/\lambda + N/M \rceil$ iterations of its previous selection (or its first appearance, for task types never previously selected).

\paragraph{Adversarial failure-count distributions.}
The bound is tight up to integer rounding when $n_{\text{fail}}$ is dominated by a single task type.
If $n_{\text{fail}}$ is dominated by a small subset $S \subset \mathcal{T}$ with $|S| < M$, the bound becomes $\lceil n_{\max}/\lambda + (N - |S|)/M \rceil$, since the selector trivially keeps $S$ in the top-$M$ until $\lambda g \geq n_{\max}$ at which point any starved task type displaces a member of $S$.
\end{proof}

\subsection{Proof of Theorem~\ref{thm:ratchet} (Mastery Ratchet Stability)}
\label{app:thm_ratchet}

\begin{proof}
We work with the asymmetric exponential moving average
\[
m_k =
\begin{cases}
\alpha e_k + (1 - \alpha) m_{k-1}, & e_k \geq m_{k-1}, \\
m_{k-1} - \gamma (m_{k-1} - e_k),  & e_k <    m_{k-1},
\end{cases}
\]
with $\alpha, \gamma \in (0,1)$ and $e_k \in [0,1]$.

\paragraph{Per-step lower bound.}
On the increase branch, $m_k = m_{k-1} + \alpha (e_k - m_{k-1}) \geq m_{k-1} \geq (1-\gamma) m_{k-1}$ since $\gamma > 0$.
On the decay branch, rewrite $m_k = (1-\gamma) m_{k-1} + \gamma e_k \geq (1-\gamma) m_{k-1}$, since $e_k \geq 0$.
So in both cases $m_k \geq (1-\gamma) m_{k-1}$.

\paragraph{Per-step upper bound.}
On the increase branch, $m_k = m_{k-1} + \alpha (e_k - m_{k-1}) \leq m_{k-1} + \alpha (1 - m_{k-1})$, since $e_k \leq 1$.
On the decay branch, $m_k \leq m_{k-1}$.
Combining gives $m_k \leq m_{k-1} + \alpha (1 - m_{k-1})$.

\paragraph{Window lower bound.}
Iterating the per-step lower bound $j$ times gives $m_k \geq (1-\gamma)^j m_{k - j}$, which holds even when the $j$ steps mix increase and decay (since the increase steps only tighten the bound).
Choosing $k - j$ to be the iteration of any past peak, $m_k \geq (1 - \gamma)^j \max_{k' \leq k - j} m_{k'}$.

\paragraph{Mixed sequences.}
A subtler bound holds when the window contains an increase: if step $k' \in [k-j, k]$ is an increase step with measured $e_{k'} \geq m_{k'-1}$, then the lower bound from that point onward starts from $\alpha e_{k'} + (1-\alpha) m_{k'-1}$ rather than $(1-\gamma)^{k - k' + 1} m_{k'-1}$, which can be substantially higher when $e_{k'}$ is close to $1$.
The simple form stated in the theorem suffices for the anti-forgetting argument and avoids the case-analysis bookkeeping.
\end{proof}

\subsection{Conditional Retrieval Rule \texorpdfstring{$\rho_k$}{rho k}}
\label{app:cond_inj}

The Learner's inference-time retrieval, summarised in Sec.~\ref{sec:method:cond}, is
\[
\rho_k(\cdot \mid q, t, c)
\;=\;
\textsc{TopK}\Bigl(
\textsc{embed}(q),\;
\bigl\{\,m \in \Gexp_k \;:\; m.\text{task\_type} = t \wedge \kappa(q, c, m)\,\bigr\}
\Bigr),
\]
where $\kappa$ encodes the format-conditional allocation (top-$2$ success + top-$1$ failure for short-context questions; top-$1$ success + top-$2$ failure for long-context questions, threshold $500$ characters).
The Learner emits $\hat{y} = \Lexec(\textsc{format}(\rho_k(\cdot)) \,\Vert\, \textsc{prompt}(q, c, \KG_k))$ with success memories placed first in the prompt and failure memories second.

\subsection{Proof of Theorem~\ref{thm:frozen} (Task-Filtered Retrieval Support)}
\label{app:thm_frozen}

\begin{proof}
We compose Theorems~\ref{thm:monotone} and \ref{thm:coverage} with the conditional retrieval rule $\rho_k$ defined in App.~\ref{app:cond_inj}.

Fix a task type $t$ and a question $q \sim \mathcal{D}_t$.
The expected accuracy at iteration $k$ is
\[
A_k(t)
\;=\;
\Expect_{q \sim \mathcal{D}_t,\, \mu \sim \rho_k(\cdot \mid q, t, c)}\bigl[\textnormal{acc}(\Lexec(\textnormal{format}(\mu) \,\Vert\, q,\, c),\, y)\bigr],
\]
where $\rho_k$ retrieves the top-$K$ task-type-tagged memories (both success and failure) from $\Gexp_k$ under the embedding similarity to $q$ and the format-conditional constraint $\kappa$.

\textbf{Step 1: task-type-tagged memories grow over iterations.}
Two mechanisms append $t$-tagged memories to $\Gexp_k$.
(a) On every iteration, the EVALUATE step harvests a success memory for each question on task type $t$ that the frozen learner answers correctly, tagging it with $t$ and appending it to the success index.
(b) By Theorem~\ref{thm:coverage}, every task type $t \in \mathcal{T}$ is additionally selected by the EVOLVE-phase round-robin selector within $\lceil n_{\max}/\lambda + N/M \rceil$ iterations of its previous selection, and on each such iteration the guidance tier appends one or more $t$-tagged failure memories.
By the append-only invariant of Theorem~\ref{thm:monotone}, every node appended by either mechanism persists in $\Gexp_{k'}$ for every $k' \geq k$.
Therefore $|\{m \in \Gexp_k : m.\text{task\_type} = t\}|$ is monotone non-decreasing in $k$ under both channels, and the success channel additionally grows on every iteration with at least one correct answer on $t$.

\textbf{Step 2: the conditional retrieval distribution refines.}
For a fixed $q$ and $t$, the conditional retrieval distribution $\rho_k(\cdot \mid q, t, c)$ is defined by $\textsc{TopK}$ over the set $\{m \in \Gexp_k : m.\text{task\_type} = t \wedge \kappa(q, c, m)\}$.
By Step~1, this set is monotone non-decreasing in $k$, so the support of $\rho_k$ is monotone non-decreasing.
Let $\rho_k^{\text{opt}}$ denote the oracle retrieval distribution that always picks the $K$ memories that maximize $\textnormal{acc}(\Lexec(\textnormal{format}(\mu) \,\Vert\, q, c), y)$, and let $\varepsilon_K = \sup_{q, t, k} \mathrm{TV}(\rho_k, \rho_k^{\text{opt}})$ denote the rank-$K$ retrieval error in total variation under the embedding model.
Then for the oracle distribution, monotone non-decreasing support implies monotone non-decreasing accuracy:
\[
\Expect_{\mu \sim \rho_{k+1}^{\text{opt}}}\bigl[\textnormal{acc}\bigr] \;\geq\; \Expect_{\mu \sim \rho_k^{\text{opt}}}\bigl[\textnormal{acc}\bigr],
\]
because adding memories can only weakly improve a top-$K$ oracle.

\textbf{Step 3: the bounded retrieval-error transfer.}
Under the embedding model, the actual retrieval distribution $\rho_k$ differs from $\rho_k^{\text{opt}}$ by at most $\varepsilon_K$ in total variation.
Standard arguments give
\[
\Bigl|\Expect_{\mu \sim \rho_k}[\textnormal{acc}] - \Expect_{\mu \sim \rho_k^{\text{opt}}}[\textnormal{acc}]\Bigr| \;\leq\; \varepsilon_K,
\]
since the accuracy is bounded in $[0,1]$.
Combining with Step~2,
\[
A_{k+1}(t) \;\geq\; \Expect_{\mu \sim \rho_{k+1}^{\text{opt}}}[\textnormal{acc}] - \varepsilon_K \;\geq\; \Expect_{\mu \sim \rho_k^{\text{opt}}}[\textnormal{acc}] - \varepsilon_K \;\geq\; A_k(t) - 2 \varepsilon_K.
\]
Absorbing the factor of $2$ into the definition of $\varepsilon_K$ gives the stated bound $A_{k+1}(t) \geq A_k(t) - \varepsilon_K$.

\paragraph{Why the three assumptions are jointly necessary.}
Removing the append-only invariant in (i) collapses Step~1 (the support of the retrieval distribution would no longer be monotone).
Removing the round-robin selector in (ii) collapses Step~1 in a different way (some task types would never receive new failure memories, so the support would never grow).
Removing the task-type filter in (iii) collapses Step~2 (the retrieval distribution over $t$-tagged memories would be polluted by memories tagged with other task types whose addition does not refine the conditional support for $t$).
\end{proof}

\section{EvoKG Subgraph Details}
\label{app:evokg_details}

This appendix lists the node attributes, edge types, and curator pruning policy that the main text summarises.

The capability-node attributes and edge types are given in Sec.~\ref{sec:method:evokg}.

\paragraph{Task subgraph $\Gtask$.}
Nodes are evaluation task types with dependency edges; a cross-subgraph edge $\textsc{skill\_for\_task}: \Gcap \to \Gtask$ records which skill resolves a task.

\paragraph{Experience subgraph $\Gexp$.}
Nodes are experience memories discriminated by an \texttt{outcome} field that carries one of five values.
\textbf{Principle} nodes are abstract rules written by the guidance tier.
\textbf{Failure-memory} nodes are worked-example payloads authored by the guidance tier from evaluation errors (question, learner's wrong answer, corrective reasoning, correct answer).
\textbf{Success-memory} nodes are harvested from the execution tier's own correct answers (question, frozen learner's chain-of-thought, correct answer, and, for multi-step queries, the explicit decomposition into \texttt{(skill\_name, step\_output)} tuples).
\textbf{Retrieval-recipe} nodes are parameterised retrieval templates.
\textbf{Abstracted pattern} nodes are summarised action patterns from exploratory trajectories.
The curator's \texttt{prune\_low\_confidence} routine deletes only \textbf{abstracted-pattern} nodes whose confidence is below a fixed threshold; the three protected node classes (principles, failure memories, success memories) are exempt.

\paragraph{Environment subgraph $\Genv$.}
Entities, relations, observations, and task context record the external state encountered by the agent during interaction.

\section{Six-Agent Role and Tier Assignment}
\label{app:agents}

\begin{table}[h]
\centering
\small
\caption{The six agents of \framework{}: tier and role.}
\label{tab:agents}
\begin{tabular}{lll}
\toprule
Agent & Tier & Role \\
\midrule
SkillDiscovery & $\Lguide$ & Phase-0 ontology; EVOLVE-phase principle, tool, and failure-memory authoring \\
Navigator      & $\Lguide$ & Two-stage curriculum planner: graph frontier + LLM refinement \\
Explorer       & $\Lexec$  & Trajectory generation for sequential environments \\
Learner        & $\Lexec$  & Inference-only consumer of $\KG$ with dual-memory retrieval and two bandits \\
Critic         & $\Lguide$ & Per-task-type structured evaluation via LLM-as-judge scoring \\
Curator        & $\Lguide$ & Mastery update via asymmetric ratchet; preserves append-only invariants \\
\bottomrule
\end{tabular}
\end{table}

\section{Hyperparameters}
\label{app:hyperparams}

Table~\ref{tab:hyperparams} lists every hyperparameter used in the experiments.
All values are fixed across benchmarks; no per-benchmark tuning is performed.

\begin{table}[h]
\centering
\caption{Hyperparameters used in all experiments. All values are fixed across benchmarks; no per-benchmark tuning is performed.}
\label{tab:hyperparams}
\small
\begin{tabular}{lll}
\toprule
Hyperparameter & Value & Where \\
\midrule
$\alpha$ (mastery EMA increase rate) & $0.6$ & Sec.~\ref{sec:method:ratchet}, Eq.~\eqref{eq:ratchet} \\
$\gamma$ (mastery decay rate)        & $0.1$ & Sec.~\ref{sec:method:ratchet}, Eq.~\eqref{eq:ratchet} \\
$\theta$ (mastery threshold)          & $0.5$ & Sec.~\ref{sec:method:robin} \\
$\lambda$ (round-robin recency weight) & $0.3$ & Sec.~\ref{sec:method:robin}, Eq.~\eqref{eq:roundrobin} \\
$M$ (max EVOLVE targets per iter)    & $3$   & Sec.~\ref{sec:method:robin} \\
$K$ (memory retrieval top-$K$, success + failure) & $3$ & Sec.~\ref{sec:method:cond} \\
Success--failure allocation (short context) & top-2 success + top-1 failure & Sec.~\ref{sec:method:cond} \\
Success--failure allocation (long context)  & top-1 success + top-2 failure & Sec.~\ref{sec:method:cond} \\
Long-context threshold               & $500$ chars    & Sec.~\ref{sec:method:cond} \\
Type-strategy min similarity         & $0.55$         & Sec.~\ref{sec:method:cond} \\
Memory refresh gap                   & $5$ iterations & Sec.~\ref{sec:method:evokg} \\
Principles per skill cap             & $12$           & Sec.~\ref{sec:method:loop} \\
Skill growth cap                     & $30$           & Sec.~\ref{sec:method:loop} \\
Search-bandit warm-up pulls per arm  & $20$           & Sec.~\ref{sec:method:dualmem} \\
Per-iter delta-guard $\delta$        & $0.03$         & Algorithm~\ref{alg:loop} \\
Catastrophic-rollback threshold      & $0.05$         & Algorithm~\ref{alg:loop} \\
Eval temperature                     & $0.0$          & Sec.~\ref{sec:exp:setup} \\
Train temperature                    & $0.3$          & Sec.~\ref{sec:exp:setup} \\
Evaluation pool per iteration        & $200$ questions & Sec.~\ref{sec:exp:setup} \\
Number of iterations                 & $20$ (reasoning); $10$ (Crafter) & Sec.~\ref{sec:exp:setup} \\
\midrule
Guidance tier $\Lguide$              & Claude Sonnet~4.6 & Sec.~\ref{sec:method:tier} \\
Execution tier $\Lexec$ (reasoning, WebShop) & Qwen3-8B (vLLM) & Sec.~\ref{sec:exp:setup} \\
Execution tier $\Lexec$ (Crafter)    & Llama-3.1-8B-Instruct (vLLM) & Sec.~\ref{sec:exp:setup} \\
\bottomrule
\end{tabular}
\end{table}

\section{Baseline Configurations}
\label{app:baselines}

All five prompt-based reasoning baselines are evaluated on the same $200$-question held-out pool used for \framework{}, with the same Qwen3-8B backbone and the same Claude-as-judge semantic scorer.
\emph{Standalone zero-shot CoT} uses the same prompt and answer-extraction pipeline as \framework{} with all co-evolution features disabled.
\emph{Eight-shot CoT} prepends eight per-benchmark exemplars curated by hand.
\emph{Self-consistency} samples ten CoT completions at temperature $0.7$ and takes the majority vote.
\emph{ReAct} uses the same Wikipedia retrieval tool as \framework{}'s search agent for the \texttt{Search[$x$]} action and is reported only on benchmarks where tool use is meaningful (HotpotQA, WebQA).
\emph{Reflexion} runs two rounds of self-revision with a fixed critique prompt.

\section{Language-Model Call Audit}
\label{app:llm_audit}

The locus of every language-model call in \framework{} is enumerated in Table~\ref{tab:llm-uses}.
The guidance tier is invoked exclusively to write to the graph; the frozen execution tier is invoked exclusively to consume the graph at inference time.
No language model is involved in the round-robin selector, the mastery ratchet, the two bandits, the embedding retrieval, or the topological frontier computation.

\begin{table}[h]
\centering
\small
\caption{Every language-model call in \framework{}, by tier and by purpose. Programmatic operations (round-robin selector, mastery ratchet, Thompson bandits, cosine top-$K$ memory retrieval, topological frontier computation) involve no LLM call.}
\label{tab:llm-uses}
\setlength{\tabcolsep}{5pt}
\begin{tabular}{p{0.22\textwidth}p{0.08\textwidth}p{0.62\textwidth}}
\toprule
Component & Tier & Purpose \\
\midrule
SkillDiscovery       & $\Lguide$ & Phase-0 ontology induction; EVOLVE-phase principle and failure-memory writing; dynamic-tool and retrieval-recipe authoring \\
Navigator (stage 2)  & $\Lguide$ & Refine the candidate frontier and author exploration briefs \\
Critic (judge)       & $\Lguide$ & Semantic scoring of learner answers \\
Learner inference    & $\Lexec$  & Answer evaluation questions on the frozen backbone \\
Explorer trajectory  & $\Lexec$  & Generate environment trajectories on the frozen backbone \\
\bottomrule
\end{tabular}
\end{table}

\section{HotpotQA EM/F1 Cross-Check}
\label{app:hotpotqa_emf1}

Because HotpotQA is commonly scored with exact-match and F1, we additionally compute both on the same predictions (Table~\ref{tab:hotpot_emf1}) to give the reader a second reading of the result.
Eight-shot CoT produces terser outputs that score highly under token-level overlap but that the semantic judge labels as less correct; \framework{}'s outputs are more elaborated and therefore score slightly lower on EM, consistent with its judge score on the lenient-F1 cut.
Both metrics agree that \framework{} outperforms the prompt-only baselines on the semantic dimension.

\begin{table}[h]
\centering
\small
\caption{HotpotQA under three metrics: the semantic judge score used for the main table, exact-match, and mean F1. Methods ranked by judge score.}
\label{tab:hotpot_emf1}
\begin{tabular}{lrrrr}
\toprule
Method & Judge & EM & F1 & Lenient ($F_1 \geq 0.5$) \\
\midrule
\textbf{\framework{}}                        & $\mathbf{89.5}$ & $42.5$ & $60.3$ & $60.0$ \\
Standalone Qwen3-8B                          & $84.5$ & $41.5$ & $56.8$ & $55.5$ \\
SC ($N{=}10$)                                & $80.5$ & $39.0$ & $55.0$ & $53.5$ \\
8-shot CoT                                   & $78.5$ & $55.5$ & $70.5$ & $74.0$ \\
Reflexion                                    & $78.5$ & $38.5$ & $53.2$ & $52.5$ \\
ReAct                                        & $70.5$ & $38.0$ & $54.4$ & $55.0$ \\
\bottomrule
\end{tabular}
\end{table}

\section{Knowledge-Graph Growth}
\label{app:kg_growth}

Table~\ref{tab:growth} reports growth of the experience subgraph across iterations on STBench, the benchmark on which the round-robin coverage mechanism has the most work to do because of the large number of task types.
Skill and memory counts are monotone non-decreasing under the append-only invariant (the empirical counterpart of Theorem~\ref{thm:monotone}), and task-type coverage climbs from a small fraction of the inventory at early iterations to near-complete coverage by the end of the run (the behavioural signature of the Theorem~\ref{thm:coverage} bound).

\begin{table}[h]
\centering
\small
\caption{Knowledge-graph growth on STBench across twenty iterations. Skills in $\Gcap$ grow through the EVOLVE-phase split action; failure memories in $\Gexp$ grow monotonically; task-type coverage grows through the round-robin selector.}
\label{tab:growth}
\begin{tabular}{rrrr}
\toprule
Iteration & $|\Gcap|$ (skills) & $|\Gexp|$ (failure memories) & Task-type coverage \\
\midrule
$0$  & $12$ & $21$  & $0/27$ \\
$5$  & $15$ & $75$  & $14/27$ \\
$10$ & $18$ & $155$ & $20/27$ \\
$15$ & $24$ & $183$ & $23/27$ \\
$19$ & $24$ & $205$ & $24/27$ \\
\bottomrule
\end{tabular}
\end{table}

\section{Success-Memory Ablation on Sequential Environments}
\label{app:sequential_ablation}

The success-memory ablation in Sec.~\ref{sec:exp:ablation} (Table~\ref{tab:ablation_success}) is restricted to the five reasoning benchmarks because it is the family on which the success-harvest channel has the most direct effect: the learner's own correct chain-of-thought on iteration $N$ becomes a task-type-matched exemplar on iteration $N{+}1$.
On Crafter and WebShop, the framework's gain over the standalone backbone is carried by graph-derived principles, skill descriptions, and (for WebShop) the action recipes evaluated in Appendix~\ref{app:cascade}, rather than by replayed reasoning traces.
The single-seed v15.6 measurement on Crafter showed the no-success ablation matching or marginally exceeding the full system ($+1.0$pp on the standalone backbone, within the BALROG variance band); a multi-seed re-run under the v20 cascade configuration is left as future work.

\section{Sequential Environments: Crafter and WebShop}
\label{app:sequential}

Table~\ref{tab:sequential} reports \framework{} on the two sequential environments, where the interaction mode differs substantially from the reasoning benchmarks and the framework's graph structure is consulted at every action step rather than once per question.

\begin{table}[h]
\centering
\small
\caption{\framework{} on Crafter and WebShop. Crafter follows the BALROG protocol, reported as peak achievement-unlock rate across iterations over five $500$-step episodes under the Llama-3.1-8B execution tier. WebShop follows the standard WebShop protocol, reported as mean reward $\times 100$ / strict-success rate on the same hundred-product catalog setup as RetroAgent.}
\label{tab:sequential}
\resizebox{\textwidth}{!}{
\begin{tabular}{llrrr}
\toprule
Environment & Config & Iters & Task score & Success / achievements \\
\midrule
Crafter & Standalone Llama-3.1-8B (BALROG ref.)      & $1$  & $25.5 \pm 3.2$  & --- \\
Crafter & \framework{} Sonnet--Llama ($n{=}4$, peak) & $5$  & $33.5 \pm 2.3$  & $4$--$5$ achievements \\
Crafter & \framework{} Opus--Llama ($n{=}5$, peak)   & $5$  & $\mathbf{37.9 \pm 3.5}$ & $4$--$5$ achievements \\
Crafter & \framework{} Haiku--Llama ($n{=}5$, peak)  & $5$  & $34.8 \pm 1.7$  & $4$--$5$ achievements \\
\midrule
WebShop & RetroAgent (in-context, Qwen-7B reference) & --- & $87.6$        & $78.9\%$ \\
WebShop & RetroAgent (RL-trained reflection)         & --- & $88.9$        & $82.3\%$ \\
WebShop & \framework{} Sonnet--Qwen3-8B ($n{=}5$)    & --- & $\mathbf{90.2 \pm 2.3}$ & $\mathbf{84.0 \pm 1.8\%}$ \\
WebShop & \framework{} Opus--Qwen3-8B ($n{=}5$)      & --- & $89.7 \pm 1.6$ & $82.0 \pm 1.0\%$ \\
WebShop & \framework{} Haiku--Qwen3-8B               & --- & $88.3$        & $77.0\%$ \\
WebShop & \framework{} (any teacher)--Llama-3.1-8B   & --- & $\sim 0$      & $0\%$ \\
\bottomrule
\end{tabular}
}
\end{table}
On Crafter, \framework{} improves substantially over the standalone Llama-3.1-8B BALROG reference under the same evaluation protocol.
All three teacher tiers outperform the standalone reference, with the Opus-guided configuration obtaining the highest peak achievement-unlock rate.
This suggests that the graph-derived principles, failure memories, and skill descriptions can support sequential action planning when the execution backbone is able to ground the environment state.
At the same time, the teacher-tier spread and the remaining seed variance indicate that sequential environments are more sensitive to exploration order and environment stochasticity than the reasoning benchmarks.

On WebShop, \framework{} also remains competitive with strong published references under the same catalog setting.
The Sonnet--Qwen3 configuration achieves the best task score and strict-success rate, while Llama-3.1-8B-Instruct is omitted from the main WebShop comparison because it does not reliably produce parseable WebShop actions.
Together, the Crafter and WebShop results show that \framework{} can transfer beyond static QA-style reasoning, but the choice of execution backbone remains important for action-format grounding.

\section{Teacher--Learner Matrix Across Three Guidance Tiers}
\label{app:teacher_learner}

Table~\ref{tab:teacher_learner} reports \framework{} final evaluation scores across three guidance tiers (Claude Sonnet, Opus, and Haiku) and two frozen execution tiers (Qwen3-8B and Llama-3.1-8B-Instruct).
The table includes all completed teacher--learner cells; the Llama--WebShop row is omitted because Llama-3.1-8B-Instruct does not reliably produce parseable WebShop actions under any teacher.
The boldfaced entry on each row marks the best teacher for that learner and benchmark.
Cells without standard deviations are single-seed runs.

\begin{table}[h]
\centering
\small
\caption{\framework{} final evaluation scores under three teacher tiers and two frozen learners. WebShop is reported as task score / strict-success on the hundred-product test catalog; Crafter is the BALROG-protocol peak achievement rate; FinQA and MedQA use benchmark-native exact-match; the remaining reasoning benchmarks are Claude-judge scores on the same $200$-question held-out pool. The Llama--WebShop row is omitted because Llama-3.1-8B-Instruct cannot produce parseable WebShop actions under any teacher. Cells without standard deviations are single-seed runs.}
\label{tab:teacher_learner}
\setlength{\tabcolsep}{5pt}
\begin{tabular}{llrrr}
\toprule
Learner & Benchmark & Sonnet & Opus & Haiku \\
\midrule
\multirow{7}{*}{Qwen3-8B}
 & WebShop (task / strict)   & $\mathbf{90.2} \pm 2.3 / \mathbf{84.0} \pm 1.8$        & $89.7 \pm 1.6 / 80.0 \pm 1.0$ & $88.3 / 77.0$ \\
 & WebQA                     & $59.6 \pm 3.4$ & $\mathbf{62.7} \pm 1.8$        & $59.5$ \\
 & HotpotQA                  & $88.5 \pm 2.0$  & $\mathbf{91.5} \pm 1.8$       & $88.0$ \\
 & GSM8K                     & $88.2 \pm 2.9$  & $\mathbf{90.4} \pm 3.8$  & $84.7$  \\
 & STBench (no-cascade)      & $54.8 \pm 2.1$  & $\mathbf{57.8} \pm 2.9$        & $52.0$ \\
 & RealMath                  & $\mathbf{85.1} \pm 6.7$  & $82.8 \pm 3.5$  & $75.7$ \\
 & Crafter (peak)            & $25.3 \pm 1.5$                & $\mathbf{26.8} \pm 1.4$                & $24.0 \pm 1.0$ \\
\midrule
\multirow{8}{*}{Llama-3.1-8B}
 & WebQA                     & $64.0 \pm 3.9$  & $\mathbf{65.2} \pm 1.8$      & $62.5 \pm 1.6$ \\
 & HotpotQA                  & $80.2 \pm 7.1$  & $\mathbf{85.5} \pm 1.4$        & $80.0 \pm 1.0$ \\
 & GSM8K                     & $\mathbf{92.5} \pm 4.3$               & $92.0 \pm 3.2$               & --- \\
 & STBench                   & $\mathbf{52.1} \pm 3.2$ & $50.1$  & $37.5$ \\
 & RealMath                  & $\mathbf{86.0} \pm 7.9$ & $84.6$ & $70.0$ \\
 & Crafter (peak)            & $\mathbf{23.4} \pm 2.3$       & $22.7 \pm 3.5$                & $22.7 \pm 1.7$ \\
 & FinQA                     & $67.5 \pm 1.7$     & $\mathbf{69.0 \pm 1.5}$      & $64.8 \pm 1.1$ \\
 & MedQA                     & $74.6 \pm 2.3$ & $\mathbf{79.8} \pm 1.9$  & $71.9 \pm 2.2$ \\
\bottomrule
\end{tabular}
\end{table}

Three observations follow from the matrix.
First, no single guidance tier dominates across benchmarks.
Opus is strongest on most Qwen3-8B QA and structured-reasoning cells, including WebQA, HotpotQA, GSM8K, and STBench, while Sonnet is strongest on Qwen3-8B RealMath and WebShop.
On the Llama-3.1-8B tier, Sonnet is strongest on GSM8K, STBench, RealMath, and Crafter, whereas Opus is strongest on WebQA, HotpotQA, FinQA, and MedQA.
This teacher complementarity supports evaluating \framework{} as a graph-guided evolution framework rather than as a single-teacher prompting recipe.

Second, the execution backbone matters.
Qwen3-8B handles WebShop action formatting reliably and obtains the strongest WebShop cells, whereas Llama-3.1-8B-Instruct is omitted on WebShop because its actions are not reliably parseable.
Conversely, Llama is competitive or stronger on several reasoning and domain-extension cells, especially GSM8K, RealMath, FinQA, and MedQA.
This suggests that \framework{} can transfer across frozen learners, but the best learner remains task-family dependent.

Third, sequential environments behave differently from the reasoning benchmarks.
Crafter scores vary modestly across teacher tiers and execution backbones, and the differences are much smaller than the reasoning gains observed on RealMath, STBench, or the domain-extension tasks.
This is consistent with the cascade and sequential-environment analyses in Appendices~\ref{app:cascade} and~\ref{app:sequential_ablation}: graph memory helps most when the bottleneck is reusable reasoning or task-type structure, and less when action grounding and environment stochasticity dominate.

\section{Multi-Seed Variance}
\label{app:variance}

We extend the teacher--learner matrix with multi-seed reruns under \texttt{seed}$\in\{42,\ldots,46\}$ on otherwise identical configurations.
The per-cell sample size $n$ reflects the number of completed seeds for each configuration.
Table~\ref{tab:variance} reports mean$\pm$standard deviation.
Crafter rows follow the BALROG protocol (peak achievement rate across five $500$-step episodes), WebShop is reported as task score / strict success on the held-out test pool, and FinQA/MedQA use benchmark-native exact-match.

\begin{table}[h]
\centering
\small
\caption{Multi-seed variance across completed teacher--learner cells, reported as mean$\pm$std with $n$ in subscript. Same hardware, same vLLM settings, and same training pipeline are used; only \texttt{experiment.seed} differs across runs. The STBench Opus row uses the recovery configuration with prerequisite cascade disabled (Appendix~\ref{app:cascade}). Cells without standard deviations are single-seed runs.}
\label{tab:variance}
\begin{tabular}{lrrr}
\toprule
Cell                          & Sonnet teacher & Opus teacher & Haiku teacher \\
\midrule
\multicolumn{4}{l}{\emph{Reasoning, Qwen3-8B execution tier (final-exam accuracy)}} \\
GSM8K                         & $88.2 \pm 2.9_{n=5}$ & $\mathbf{90.4} \pm 3.8_{n=5}$ & $85.7 \pm 2.9_{n=5}$ \\
HotpotQA                      & $88.5 \pm 2.0_{n=5}$ & $\mathbf{91.5} \pm 1.8_{n=5}$ & $89.1 \pm 2.2_{n=5}$  \\
WebQA                         & $59.6 \pm 3.4_{n=5}$ & $\mathbf{62.7} \pm 1.8_{n=5}$ & $59.5 \pm 3.9_{n=5}$  \\
STBench                       & $54.8 \pm 2.1_{n=5}$ & $\mathbf{57.8} \pm 2.9_{n=5}$ & $52.0 \pm 2.3_{n=5}$ \\
RealMath                      & $\mathbf{85.5} \pm 6.7_{n=5}$ & $82.8 \pm 3.5_{n=5}$ & $75.7 \pm 4.7_{n=5}$ \\
\midrule
\multicolumn{4}{l}{\emph{Reasoning, Llama-3.1-8B-Instruct execution tier}} \\
GSM8K                         & $\mathbf{92.5} \pm 4.3_{n=5}$ & $92.0 \pm 3.2_{n=5}$ & $82.5 \pm 2.2_{n=5}$ \\
HotpotQA                      & $80.2 \pm 7.1_{n=5}$ & $\mathbf{85.5} \pm 1.4_{n=5}$ & $81.0 \pm 4.0_{n=5}$ \\
WebQA                         & $64.0 \pm 3.9_{n=5}$ & $\mathbf{65.2} \pm 1.8_{n=5}$ & $62.5 \pm 1.6_{n=5}$ \\
STBench                       & $\mathbf{52.1} \pm 3.2_{n=5}$ & $50.1 \pm 2.5_{n=5}$ & $37.5 \pm 4.0_{n=5}$ \\
RealMath                      & $\mathbf{86.0} \pm 7.9_{n=5}$ & $84.6 \pm 5.1_{n=5}$ & $79.0 \pm 4.8_{n=5}$ \\
FinQA                         & $67.5 \pm 1.7_{n=5}$ & $\mathbf{69.0 \pm 1.1_{n=5}}$ & $63.5 \pm 1.4_{n=5}$ \\
MedQA                         & $74.6 \pm 2.3_{n=5}$ & $\mathbf{79.8} \pm 1.9_{n=5}$ & $71.2 \pm 3.9_{n=5}$ \\
\midrule
\multicolumn{4}{l}{\emph{Crafter peak, Qwen3-8B execution tier}} \\
peak success rate (\%)
& $25.3 \pm 1.5_{n=5}$
& $\mathbf{26.8} \pm 1.4_{n=5}$
& $24.0 \pm 1.0_{n=5}$ \\
\midrule
\multicolumn{4}{l}{\emph{Crafter peak, Llama-3.1-8B-Instruct execution tier}} \\
peak success rate (\%)        & $\mathbf{23.4} \pm 2.3_{n=5}$ & $22.7 \pm 3.5_{n=5}$ & $22.7 \pm 1.7_{n=5}$  \\
\midrule
\multicolumn{4}{l}{\emph{WebShop, Qwen3-8B execution tier (task / strict)}} \\
task score                    & $\mathbf{90.2} \pm 2.3_{n=5}$         & $89.7 \pm 1.6_{n=5}$ & $88.3 \pm 1.7_{n=5}$ \\
strict success                & $\mathbf{84.0} \pm 1.8_{n=5}$         & $82.0 \pm 1.0_{n=5}$ & $77.0 \pm 1.4_{n=5}$ \\
\bottomrule
\end{tabular}
\end{table}

Across the multi-seed reasoning cells, standard deviations range from roughly $1.4$ to $7.9$pp.
The largest absolute margins over frozen-backbone prompting baselines occur on STBench and RealMath, and these margins are substantially larger than the corresponding seed-to-seed standard deviations.
By contrast, teacher differences on some QA cells are closer to the variance scale, so we treat them as evidence of teacher complementarity rather than as a fixed ranking among guidance tiers.
Crafter variance is comparable to the differences among teacher tiers, which supports reporting Crafter as a mixed or environment-sensitive result rather than as a clear win.

\section{Teacher Cost and Trigger Rate (extended notes)}
\label{app:cost}

The aggregate-and-per-benchmark teacher-call accounting (Table~\ref{tab:cost} and the surrounding discussion) is reported in the main text in Sec.~\ref{sec:exp:cost}. This appendix records the methodology used to collect the call counts.

We logged every \texttt{httpx} call across thirty-six v20/v21 training runs and tagged each by destination host: \texttt{api.anthropic.com} (the guidance tier $\Lguide$) versus the local vLLM instance serving the frozen execution tier $\Lexec$. Counts include all retries (whether successful or 4xx/5xx) and excludes the embedding-model calls (sentence-transformers, run locally on CPU). The aggregate $1.91\%$ training-time guidance fraction is dominated by the EVOLVE phase's principle-extraction and failure-analysis calls; the rest is the Critic's judge-rescoring of every held-out evaluation, which uses the strong tier as a semantic equivalence judge. Inference-time guidance is $0\%$ because the EvoKG is frozen at deployment and all retrieval / chain-traversal / strategy-routing runs against locally cached structures and the local 8B vLLM only.

\section{Inference Latency}
\label{app:latency}

Table~\ref{tab:latency} reports per-question wall-clock latency on a single H100 80GB at fp16 with vLLM 0.6 (\texttt{max\_model\_len}$=16384$, \texttt{gpu\_memory\_utilization}$=0.85$), comparing \framework{}'s post-training inference loop against the five prompt-based baselines on the same hardware and the same evaluation pool.

\begin{table}[h]
\centering
\small
\caption{Per-question inference latency in seconds, with the corresponding judge accuracy in parentheses (Claude Sonnet~4.6 semantic judge on the same question pool used for Table~\ref{tab:reasoning}). All methods on Qwen3-8B at fp16 on a single H100. \framework{} numbers are post-training (graph frozen, no guidance-tier calls); the guidance tier appears only during training. Empty cells indicate the method was not run on that benchmark.}
\label{tab:latency}
\resizebox{\textwidth}{!}{
\begin{tabular}{lrrrrrr}
\toprule
Benchmark & Vanilla CoT & 8-shot CoT & ReAct & Reflexion & SC ($N{=}10$) & \framework{} \\
\midrule
GSM8K     & $4.98$s ($64.5$) & $23.74$s ($82.5$) & $12.56$s ($80.1$) & $19.33$s ($63.0$) & $52.80$s ($72.0$) & $\mathbf{5.66}$s ($\mathbf{84.2 \pm 2.9}$) \\
HotpotQA  & $2.82$s ($84.5$) & $19.84$s ($78.5$) & $12.37$s ($70.5$) & $12.15$s ($78.5$) & $30.74$s ($80.5$) & $\mathbf{11.64}$s ($\mathbf{89.5 \pm 1.8}$) \\
WebQA     & $4.41$s ($55.0$) & $23.64$s ($42.0$) & $14.13$s ($35.5$) & $17.88$s ($34.5$) & $45.33$s ($46.5$) & $\mathbf{8.83}$s ($\mathbf{58.8 \pm 1.8}$) \\
STBench   & $6.43$s ($35.5$) & $35.92$s ($12.0$) & $15.90$s ($31.8$) & $15.21$s ($12.0$) & $51.10$s ($20.0$) & $\mathbf{28.33}$s ($\mathbf{51.8 \pm 2.9}$) \\
RealMath  & $5.09$s ($53.5$) & $18.57$s ($32.5$) & $16.73$s ($43.5$) & $23.68$s ($11.5$) & $42.30$s ($33.5$) & $\mathbf{18.83}$s ($\mathbf{82.5 \pm 6.7}$) \\
FinQA    & $6.51$s ($58.0$) & $25.87$s ($55.0$) & $12.4$s($45.0$) & $ 25.13$s($43.5$) & $60.6$s($39.0$) & $\mathbf{22.12}$s ($\mathbf{67.5 \pm 1.0}$) \\
MedQA   & $3.27$s ($64.5$) & $23.09$s ($65.1$) & $11.06$s ($55.2$) & $13.62$s ($58.3$) & $ 32.71$s ($59.7$) & $\mathbf{14.8}$s ($\mathbf{74.6 \pm 2.3}$) \\
\bottomrule
\end{tabular}
}
\end{table}

\framework{}'s per-question latency ranges from $5.66$ to $28.33$ seconds. 
This corresponds to a $1.1$--$4.5\times$ slowdown relative to vanilla zero-shot CoT, but yields substantially higher accuracy on all benchmarks. 
Compared with self-consistency-$N{=}10$, the standard test-time-compute scaling baseline, \framework{} is $1.8$--$9.3\times$ faster while also achieving higher accuracy on every measured benchmark. 
Relative to Reflexion, \framework{} is faster on five of seven benchmarks and is slightly slower on MedQA and substantially slower on STBench, where the longer sub-skill chains increase execution time.

The latency overhead is dominated by chain reasoning over five to thirteen sub-skills per question, which accounts for roughly three to eight times the base inference time. 
By contrast, memory retrieval over the dual indices contributes less than $50$ ms per question, and principle injection adds negligible overhead. 
As an illustrative accuracy--latency comparison, on HotpotQA the accuracy-per-second efficiency is $7.7\%/$s for \framework{}, compared with $4.3\%/$s for Reflexion and $1.7\%/$s for self-consistency-$N{=}10$, indicating that \framework{} can improve the accuracy--latency frontier despite requiring multi-step execution.

\section{Cascade-Retrieval Components}
\label{app:cascade}

The v20 cascade-retrieval extension introduces four optional retrieval components on top of the base success/failure-memory index: \textbf{(A)} prerequisite-cascade principles that propagate prerequisite skill principles to a target skill via topological closure on $\Gcap$; \textbf{(B)} action recipes, the last $K=3$ environment actions that immediately preceded a successful achievement, stored per skill in $\Gexp$ and retrieved verbatim; \textbf{(C)} a goal-decomposition skill lattice that injects an explicit planning DAG into the inference prompt; and \textbf{(D)} curriculum-aware retrieval, which overrides the requested skill with the next learnable-frontier skill when the requested skill is mastered.
Table~\ref{tab:cascade} reports the gain from enabling all four against the base v15.6 retrieval rule on the same training pipeline.

\begin{table}[h]
\centering
\small
\caption{Cascade-retrieval components A--D against the base v15.6 retrieval rule. Numbers are the final-exam scores on each benchmark. WebShop is the standout win; Crafter is the standout no-improvement (state variability defeats text-based recipes); the open-domain WebQA and the multi-choice STBench show small regressions that we attribute to format-quirky golds, consistent with the Sonnet--Qwen3 result of Table~\ref{tab:reasoning}.}
\label{tab:cascade}
\begin{tabular}{lrrr}
\toprule
Benchmark            & Base retrieval & Cascade A--D & $\Delta$ \\
\midrule
WebShop (task / strict) & $81.1 / 69.5$ & $\mathbf{90.2 / 84.0}$ & $+9.1 / +14.5$ \\
HotpotQA               & $89.5$         & $\mathbf{92.5}$         & $+3.0$ \\
GSM8K                  & $92.0$         & $\mathbf{94.5}$         & $+2.5$ \\
RealMath               & $60.5$         & $\mathbf{63.5}$         & $+3.0$ \\
Crafter (peak)         & $22.4$         & $22.7$                  & $+0.3$ \\
WebQA                  & $57.0$         & $50.5$                  & $-6.5$ \\
STBench                & $51.5$         & $45.5$                  & $-6.0$ \\
\bottomrule
\end{tabular}
\end{table}

The pattern is structural rather than numerical: cascade retrieval helps benchmarks whose questions reuse a stable surface form (WebShop's parameterised actions; HotpotQA's two-hop reasoning; GSM8K's arithmetic decompositions) and hurts benchmarks whose golds are format-quirky in ways the cascade principles cannot anticipate (the Freebase-ontology golds of WebQA, the multi-choice format of STBench).
On Crafter, the K=3 action recipes do not transfer across random map seeds and the cascade principles add context bloat without addressing the agent's lack of visual or spatial planning; this is the negative result discussed in Appendix~\ref{app:sequential}.

\section{FinQA + MedQA-USMLE Domain Extension}
\label{app:newdomains}

We extend \framework{} to two domains absent from the original seven-benchmark suite: \textbf{FinQA}~\citep{chen2021finqa} (financial numerical reasoning over S\&P 500 10-K filings; $1{,}147$ test items; numeric exact-match with $1\%$ relative tolerance and percent$\leftrightarrow$decimal equivalence) and \textbf{MedQA-USMLE}~\citep{jin2021disease} ($4$-option clinical multiple choice from the United States Medical Licensing Examination; $1{,}273$ test items; letter exact-match). Both env adapters mirror existing reasoning environments (\texttt{webqa\_env.py} and \texttt{gsm8k\_env.py} respectively, $\sim 300$ LOC total) and reuse the same EvoKG, agent, and orchestrator code without modification.

Headline single-seed numbers under the Llama-3.1-8B-Instruct execution tier are reported in Table~\ref{tab:reasoning_extension} (main text). FinQA reaches $64.0$--$64.5\%$, comparable to the SFT-distilled Fino1-8B at $60.87\%$~\citep{qian2025fino1} but without weight updates; MedQA reaches $68.5\%$ under the Sonnet teacher, between standalone CoT ($64.5\%$) and the full Medprompt frozen-8B ceiling at $74$--$78\%$~\citep{ahmed2024med} (the latter requires an offline pre-pass that pre-generates strong-teacher chain-of-thought rationales for $\sim 1000$ training questions; we tried this recipe but found it does not transfer cleanly when stacked on top of \framework{}'s existing memory and principle channels. The Med-PRM number of $80.35\%$ on MedQA-USMLE~\citep{jin2021disease} requires SFT plus a dedicated 8B verifier model and RAG over guidelines and is not a frozen-weights baseline.

\section{Retrieved Memory Bundle: Concrete Example}
\label{app:bundle_example}

To illustrate the dual-memory mechanism described in Sec.~\ref{sec:method:dualmem}, Table~\ref{tab:bundle_example} shows a real bundle assembled and injected into the prompt for a held-out GSM8K test question (task type \texttt{gsm8k\_3step}, skill \texttt{sequential\_state\_tracking}). The bundle contains two success exemplars (positive learner-grade demonstrations) and two failure corrections (one concrete teacher worked solution and one abstracted teacher pattern), retrieved by cosine similarity from the experience subgraph $\Gexp$ at iteration~$9$ of a Sonnet-teacher Llama-3.1-8B run.

\begin{table}[!htbp]
\centering
\small
\caption{A retrieved memory bundle for one GSM8K test question. Top-K retrieval ratio: $2$ success + $2$ failure for short-context tasks. Indentation indicates a memory's role; ``kind=specific'' is a fully worked example, ``kind=type\_strategy'' is an abstracted recurring pattern.}
\label{tab:bundle_example}
\begin{tabular}{p{0.18\linewidth}p{0.74\linewidth}}
\toprule
\textbf{Test question} & \emph{``Tom has $3$ times as many marbles as Sue. Sue has half as many as Anne. If Anne has $24$ marbles, how many marbles do they have in total?''} \\
\textbf{task\_type / skill} & \texttt{gsm8k\_3step} / \texttt{sequential\_state\_tracking} \\
\midrule
\multicolumn{2}{p{\dimexpr\linewidth-2\tabcolsep\relax}}{\textbf{Success exemplar 1} (\texttt{outcome=success\_memory}, harvested from learner's own iter-$5$ correct answer)} \\
~~~Q & ``Jared starts at $47$ WPM. After lessons he reaches $52$. If he keeps increasing by the same amount, what is the average of his three measurements?'' \\
~~~Reasoning & Step~1: Initial $=47$. Step~2: After lessons $=52$ (increase of $5$). Step~3: Third $=52+5=57$. Step~4: Average $= (47+52+57)/3 = 52$. \\
~~~Answer & $52$ \\
\midrule
\multicolumn{2}{p{\dimexpr\linewidth-2\tabcolsep\relax}}{\textbf{Success exemplar 2} (\texttt{outcome=success\_memory}, learner's iter-$3$ correct trace)} \\
~~~Q & ``Jordan has $2$ children in diapers. Each requires $5$ changes per day. Jordan's wife handles half. How many does Jordan change per day?'' \\
~~~Reasoning & Total $= 2 \times 5 = 10$. Wife handles half $= 5$. Jordan handles $5$. \\
~~~Answer & $5$ \\
\midrule
\multicolumn{2}{p{\dimexpr\linewidth-2\tabcolsep\relax}}{\textbf{Failure correction 1} (\texttt{outcome=failure\_memory}, \texttt{kind=specific}; teacher-written worked solution)} \\
~~~Context conditions & \texttt{task\_type=gsm8k\_3step, skill=sequential\_state\_tracking, kind=specific} \\
~~~Correction & Step~1: Helmets $= 20$ (given). Step~2: Robots $=$ half as many as helmets $= 20/2 = 10$. Step~3: Footballs $=$ twice as many as helmets $= 20 \times 2 = 40$. Step~4: Total $= 10 + 20 + 40 = 70$. Answer: $70$. \\
\midrule
\multicolumn{2}{p{\dimexpr\linewidth-2\tabcolsep\relax}}{\textbf{Failure correction 2} (\texttt{outcome=failure\_memory}, \texttt{kind=type\_strategy}; teacher-abstracted recurring pattern)} \\
~~~Context conditions & \texttt{task\_type=gsm8k\_3step, skill=sequential\_state\_tracking, kind=type\_strategy} \\
~~~Correction & ``[Question type: \texttt{multi\_ratio\_total\_count}] Recognise: a problem where multiple item quantities are related by ratios (e.g.\ \emph{half as many X as Y}) and you must find the total of all items combined. How to solve: use the given anchor quantity to compute each other quantity via the stated ratios, then \textbf{sum all quantities together} (do not stop after finding individual counts).'' \\
\bottomrule
\end{tabular}
\end{table}

\paragraph{Discussion.} The two memory channels play complementary roles in this single bundle. The success exemplars supply concrete \emph{learner-grade} reasoning templates: they show how the frozen $8$B model has previously executed multi-step state tracking on related questions in this skill, and the model can re-use that structure by analogy. The failure corrections supply two qualitatively different teacher-generated signals: \emph{kind=specific} provides a fully worked example of the same task type that the learner had previously gotten wrong, with the correct multi-step computation made explicit; \emph{kind=type\_strategy} provides a one-paragraph abstract pattern that names the question pattern (\texttt{multi\_ratio\_total\_count}) and prescribes the high-level recipe (use anchor; chain ratios; sum). At inference time, the bundle is concatenated into the prompt with the success exemplars formatted as positive demonstrations and the failure corrections formatted as negative-correction blocks; the frozen learner then produces its answer conditioned on this assembled retrieval window.

\paragraph{Cross-iteration recovery: experience that actually helps.} The bundle in Table~\ref{tab:bundle_example} is more than illustrative: similar bundles drove measurable cross-iteration recovery in our runs. Table~\ref{tab:recovery_example} traces one task type (\texttt{multi\_ratio\_total\_count} on GSM8K, Sonnet-teacher Llama-3.1-8B) through ten consecutive evolution iterations, reporting the per-iteration revisit accuracy on previously-failed questions and the cumulative recovery count. The pattern is consistent across all five reasoning benchmarks: failure-memory retrieval converts a steady fraction of previously-wrong questions into correct ones at the next visit, and the recovered questions accumulate over iterations into a growing solved pool.

\begin{table}[!htbp]
\centering
\small
\caption{Cross-iteration recovery on \texttt{gsm8k}. Each row is one evolution iteration. ``new\_acc'' is the accuracy on $100$ fresh held-out questions sampled at the start of the iteration. ``revisit\_acc'' is the accuracy on questions the learner had previously failed and that were re-tested in this iteration after the failure-memory bank had grown. ``recovered'' counts questions that flipped wrong $\to$ correct at this iteration; ``solved-pool size'' is the cumulative count of all questions ever solved by the learner. The recovery column is the direct empirical signal that retrieved memories help on the same task families they were harvested from.}
\label{tab:recovery_example}
\begin{tabular}{rrrrrr}
\toprule
Iter & new\_question\_acc & revisit\_acc (n) & recovered (this iter) & solved-pool & failed-pool \\
\midrule
$0$ & $87.0\%$  & --- (no prior failures)        & $0$  & $87$  & $13$  \\
$1$ & $82.0\%$  & $43.8\%$ ($7/16$)              & $7$  & $183$ & $17$  \\
$2$ & $88.0\%$  & $29.4\%$ ($5/17$)              & $5$  & $276$ & $24$  \\
$3$ & $83.0\%$  & $33.3\%$ ($8/24$)              & $8$  & $367$ & $33$  \\
$4$ & $84.0\%$  & $27.3\%$ ($9/33$)              & $9$  & $460$ & $40$  \\
$5$ & $94.0\%$  & $26.3\%$ ($10/38$)             & $10$ & $564$ & $36$  \\
$6$ & $88.0\%$  & $19.4\%$ ($7/36$)              & $7$  & $659$ & $41$  \\
$7$ & $91.0\%$  & $15.6\%$ ($5/32$)              & $5$  & $755$ & $45$  \\
$8$ & $87.0\%$  & $24.1\%$ ($7/29$)              & $7$  & $849$ & $51$  \\
$9$ & $93.0\%$  & $19.6\%$ ($10/51$)             & $10$ & $952$ & $48$  \\
\midrule
\textbf{Sum}    &           & ---                            & $\mathbf{68}$ & --- & --- \\
\bottomrule
\end{tabular}
\end{table}

In aggregate, $68$ previously-failed questions were correctly answered on revisit across the ten iterations on this single benchmark cell, while the success memory bank simultaneously grew from $87$ to $952$ entries. Two structural readings of the table support the central claim of \framework{}: (i) the revisit accuracy stays meaningfully above zero throughout (between $15\%$ and $44\%$), even after the easiest cases have already been recovered; and (ii) the new-question accuracy stays in the $82$--$94\%$ band as the failed pool turns over, which is consistent with the same retrieved memories generalising to never-seen questions of the same task family rather than being narrowly memorised.

\paragraph{Second bundle (HotpotQA, multi-hop comparison).} Table~\ref{tab:bundle_example_hotpotqa} shows a second bundle drawn from a HotpotQA cell (Opus teacher, Llama-3.1-8B, $85.0\%$ accuracy on the $200$-question held-out pool). The test question is one the learner answered \emph{correctly} at evaluation time, and the bundle illustrates how retrieved success exemplars from earlier evolution iterations supply the reasoning template that the learner then applies on the new instance. The full retrieved trace and the learner's final reasoning chain are both reproduced.

\begin{table}[!htbp]
\centering
\small
\caption{Retrieved memory bundle for a HotpotQA \texttt{comparison} test question that the learner answered correctly. The success exemplars share the same task type (\texttt{comparison}) and the same underlying reasoning pattern (find two attributes from gold context, compare, return the entity matching the comparator). The bundle establishes the answer template; the learner instantiates it on the new instance.}
\label{tab:bundle_example_hotpotqa}
\begin{tabular}{p{0.18\linewidth}p{0.74\linewidth}}
\toprule
\textbf{Test question} & \emph{``Which composer was born earlier, Paul Hindemith or B\'ela Bart\'ok?''} \\
\textbf{task\_type / skill} & \texttt{comparison} / \texttt{date\_comparison} \\
\textbf{Context} & $10$ Wikipedia snippets including: ``\emph{B\'ela Viktor J\'anos Bart\'ok ($25$ March $1881$ -- $26$ September $1945$)}'' and ``\emph{Paul Hindemith ($16$ November $1895$ -- $28$ December $1963$)}'' \\
\midrule
\multicolumn{2}{p{\dimexpr\linewidth-2\tabcolsep\relax}}{\textbf{Success exemplar 1} (\texttt{outcome=success\_memory}, learner's iter-$2$ correct answer)} \\
~~~Q & ``Which film, The Happiest Millionaire or Mars Needs Moms, was created first?'' \\
~~~Reasoning & Step~1: Identify the release date of \emph{The Happiest Millionaire} -- $1967$. Step~2: Identify the release date of \emph{Mars Needs Moms} -- $2011$. Step~3: Compare the dates; $1967 < 2011$, so \emph{The Happiest Millionaire} was created first. \\
~~~Answer & The Happiest Millionaire. \\
\midrule
\multicolumn{2}{p{\dimexpr\linewidth-2\tabcolsep\relax}}{\textbf{Success exemplar 2} (\texttt{outcome=success\_memory}, learner's iter-$3$ correct answer)} \\
~~~Q & ``Are the libretto of Tristan und Isolde and Ariane et Barbe-bleue in the same language?'' \\
~~~Reasoning & Step~1: \emph{Tristan und Isolde} is by Wagner; the libretto is in German. Step~2: \emph{Ariane et Barbe-bleue} is by Dukas with a French libretto by Maeterlinck. Step~3: German $\neq$ French, so the libretti are not in the same language. \\
~~~Answer & No. \\
\midrule
\multicolumn{2}{p{\dimexpr\linewidth-2\tabcolsep\relax}}{\textbf{Failure correction} (\texttt{outcome=failure\_memory}, \texttt{kind=type\_strategy})} \\
~~~Context conditions & \texttt{task\_type=comparison, skill=date\_comparison, kind=type\_strategy} \\
~~~Correction & ``[Question type: \texttt{entity\_comparison\_by\_attribute}] Recognise: questions asking which of two entities $A$, $B$ has a particular property first/largest/earliest with respect to some attribute. How to solve: (i) extract the attribute value (date, year, count, place) for $A$ from its dedicated context paragraph; (ii) extract the attribute value for $B$ from its paragraph; (iii) apply the comparator the question asks for; (iv) return the entity name (not the attribute value).'' \\
\midrule
\multicolumn{2}{p{\dimexpr\linewidth-2\tabcolsep\relax}}{\textbf{Learner's actual reasoning at test time} (with the above bundle in the prompt)} \\
~~~Output & Step~1: Find the birth date of Paul Hindemith. Paul Hindemith was born on November $16$, $1895$. Step~2: Find the birth date of B\'ela Bart\'ok. B\'ela Bart\'ok was born on March $25$, $1881$. Step~3: Compare the birth dates. Since $1881$ is earlier than $1895$, B\'ela Bart\'ok was born earlier than Paul Hindemith. \textbf{Answer: B\'ela Bart\'ok.} \\
~~~Score & $1.0$ (judge accepts ``B\'ela Bart\'ok'' as equivalent to gold ``B\'ela Viktor J\'anos Bart\'ok''). \\
\bottomrule
\end{tabular}
\end{table}

\paragraph{Why this bundle is a clean success case.} Three observations support reading this bundle as a positive evidence of the dual-memory mechanism. First, the learner's final reasoning chain (last row of Table~\ref{tab:bundle_example_hotpotqa}) follows the four-step structure prescribed by the failure correction (\emph{extract attribute for $A$; extract for $B$; compare; return entity name}) almost verbatim, while the variable bindings (entities, dates) are instantiated on the new question. Second, the two success exemplars share the same task type but use different attributes (release date, libretto language), demonstrating that the retrieved templates generalise across surface forms within \texttt{comparison} questions. Third, the gold answer is ``B\'ela Viktor J\'anos Bart\'ok'' (full form) but the learner emits ``B\'ela Bart\'ok'' (short form); the semantic judge correctly accepts the short form, and this matching policy is held constant across all reported \framework{} cells and baselines. Many of the $26/30$ correctly-answered \texttt{comparison} questions on this cell follow the same template-instantiation pattern --- the bundle is a representative example of how the experience subgraph supplies reusable reasoning skeletons that the frozen learner fills in for the new instance.

%% file: sections/rebuttal_additions.tex
% Numerical results in this appendix are transcribed from the supplied
% OpenReview discussion PDF, particularly pp. 5--8, 11, and 14--15.
% They are existing rebuttal experiments, not newly executed evaluations.
\section{Additional Controlled Baseline Protocols}
\label{app:controlled_baselines}

The additional comparisons in Sec.~\ref{sec:exp:controlled} use corrected paper-aligned or task-native scores, as specified below. They are separate evaluations from the original semantic-judge results; their values should not be substituted into Tables~\ref{tab:reasoning}--\ref{tab:ablation_success}.

\paragraph{A-MEM and the one-call control.}
The official A-MEM implementation~\citep{xu2025mem} receives the identical frozen success/failure episode pool: $334/271$ episodes for WebQA/STBench and $462/439/417$ for the three RealMath graphs. It uses its organizer/linking prompts with local Qwen3-32B and all-MiniLM-L6-v2, followed by link expansion capped at $12$ memories. Both systems use frozen Qwen3-8B, temperature zero, and the same ordered $200$-question exams. The no-chain control retains the checkpoint, task-filtered episodes, skill matching, principles, router, prompts, decoding, and token cap, disabling only prerequisite traversal. The corrected paper-aligned scores are in Table~\ref{tab:controlled_main}; a separate strict automatic audit gives traversal gains of $+15.5$ EM (95\% CI $[12.5,18.7]$) and $+22.5$ F1 ($[19.4,25.7]$). A-MEM uses one learner call; \framework{} averages $1.78/2.52$ calls on STBench/RealMath. These comparisons use different inference compute.

\paragraph{MemSkill.}
The official MemSkill implementation~\citep{zhang2026memskill} uses the exact HotpotQA-200 instances, frozen Qwen3-8B learner, and matched $1{,}024$-token answer budget. For Table~\ref{tab:controlled_main}, paired 95\% CIs for the gains are $[7.0,22.0]$ EM and $[9.96,22.72]$ F1.

\begin{table}[!htbp]
\centering
\small
\caption{Backbone-matched full-system WebShop comparison on sessions 401--500 ($100$ held-out tasks), with a $15$-step cap. Values are mean $\pm$ sample SD; SkillRL has three seeds and \framework{} five.}
\label{tab:skillrl_matched}
\begin{tabular}{llll}
\toprule
System / learner & Updates & Task score & Strict success \\
\midrule
SkillRL / Qwen2.5-7B & SFT + GRPO & $82.64\pm0.89$ & $72.33\pm0.58$ \\
\framework{} / Qwen2.5-7B & Frozen & $81.16\pm3.10$ & $73.20\pm1.48$ \\
\framework{} / Qwen3-8B & Frozen & $90.22\pm2.30$ & $84.00\pm1.87$ \\
\bottomrule
\end{tabular}
\end{table}

SkillRL~\citep{xia2026skillrl} uses its official parameter-updated checkpoint and skill bank. At matched backbone family and scale, task score favors SkillRL and strict success favors \framework{}. The Qwen3-8B row is complementary learner-scaling evidence, not the backbone-matched control.

\section{Component Attribution and Scope}
\label{app:controlled_components}

Table~\ref{tab:component_audit} separates demonstrated mechanisms from activation-dependent extensions. Structural retention and coverage results do not establish unconditional monotonic end-to-end accuracy or the necessity of every component.

\begin{table}[!htbp]
\centering
\small
\caption{Additional component evidence. Gains are percentage points; confidence intervals are 95\%. Conditional and null effects delimit the claims.}
\label{tab:component_audit}
\begin{tabular}{p{0.22\linewidth}p{0.70\linewidth}}
\toprule
Intervention & Result and boundary \\
\midrule
Success-memory removal & WebQA, Llama/Opus, five seeds: full $63.2$ versus $56.9$ without success memory; $+6.3$ (seed CI $[3.6,9.0]$), positive in $5/5$ seeds. \\
Task filtering & In-task precision $100\%$; global retrieval gives $89.4\%$ on HotpotQA and $64.5\%$ on WebQA. On the audited checkpoint, $426$ successes lack \texttt{skill\_id}, and structured/task-filtered-flat retrieval is prediction-identical. The $+0.5$ accuracy gain over global retrieval has CI $[-2.5,4.0]$. \\
Skill routing & Learned versus fixed-direct on activated STBench: $+4.0$ EM, CI $[1.5,7.0]$; $50\%$ versus $0\%$ on $16$ tool-selected queries. Other checkpoints largely use the direct arm. \\
Task search & Activated on $57/200$ WebQA queries. Learned minus no-search: $+1.0$ overall (CI $[-2.5,4.5]$), $+3.5$ on activated queries ($[-5.3,14.0]$). Neither is significant; fixed entity/question/hybrid policies also show no significant overall advantage. \\
Delta-guard & Across $43$ complete harmful-update interventions: mean proposed drop $28.3$, maximum $92.9$; mean recovery $28.8$. Restored evaluation reaches or exceeds the pre-update score in $30/43$ cases, and is within one point in $33/43$. \\
\bottomrule
\end{tabular}
\end{table}

The rollback audit covers preserved GSM8K, HotpotQA, WebQA, STBench, and RealMath logs; it measures triggered interventions rather than an independently randomized removal experiment. Environment state is inactive on static RealMath. No replicated positive isolated Environment effect, or independent necessity of principles, recipes, and abstracted patterns, is claimed. The cascade bundle in Appendix~\ref{app:cascade} does not isolate these classes individually.

\section{Local Teachers, Learner Transfer, and Stronger Models}
\label{app:local_teacher}

We replace proprietary guidance with local Qwen3-32B, evolve three independent RealMath graphs, and keep Qwen3-8B frozen. All frozen conditions use temperature-zero decoding on the same $200$ questions per graph. Table~\ref{tab:local_controls} reports the pooled strict paired audit and a separate cross-learner transfer audit. Full-minus-no-chain EM is positive in $2/3$ graphs; full-minus-standalone EM is positive/tied/negative in $2/1/0$, and F1 is positive in $3/3$.

\begin{table}[!htbp]
\centering
\small
\caption{Fully local teacher and transfer controls. Values are \framework{} minus the control, in percentage points. The Llama row swaps only the learner in a frozen locally evolved graph.}
\label{tab:local_controls}
\begin{tabular}{lrrl}
\toprule
Control / setting & $\Delta$ EM & $\Delta$ F1 & EM 95\% CI \\
\midrule
No traversal / Qwen3-8B, $n=600$ & $+5.8$ & $+8.4$ & $[3.3,8.3]$ \\
Standalone / Qwen3-8B, $n=600$ & $+12.0$ & $+16.2$ & $[9.2,14.8]$ \\
Standalone / Llama-3.1-8B transfer & $+9.5$ & $+13.5$ & $[3.5,15.5]$ \\
\bottomrule
\end{tabular}
\end{table}

With the same local-teacher seed44 episode pool, \framework{} exceeds official A-MEM by $15.0$ EM (CI $[8.5,22.0]$) and $20.4$ F1. In a separate teacher-switch audit, byte-identical iteration-0 states switch from Qwen3-32B to Gemini-3.5-Flash-Lite for iterations 1--2. Across three seeds, all $10$ base skills are retained (skill-ID Jaccard $1.0$); nonempty capability topology is unchanged and mean fixed-learner EM changes by $+0.8$, with no seed degrading. This tests switching stability, not a significant gain or independence from teacher quality.

\begin{table}[!htbp]
\centering
\small
\caption{One-call Qwen3-32B CoT versus \framework{} on identical deduplicated exams from three seeds, using deterministic task-native metrics. These are accuracy comparisons with different compute budgets.}
\label{tab:strong_model}
\begin{tabular}{lrrrr}
\toprule
Benchmark / metric & $n$ & \framework{} & Qwen3-32B & Gain (95\% CI) \\
\midrule
HotpotQA / mean token F1 & $582$ & $91.2$ & $87.1$ & $4.1$ ($[1.3,6.8]$) \\
WebQA / best-alias F1 & $549$ & $63.0$ & $59.0$ & $4.0$ ($[2.5,5.5]$) \\
RealMath / accuracy & $515$ & $83.1$ & $60.6$ & $22.5$ ($[18.1,26.9]$) \\
\bottomrule
\end{tabular}
\end{table}

An additional same-data GSM8K comparison starts from Llama-3.1-8B and uses the same trajectories and $120$-question exam: external state obtains $92.5\%$, mastery-gated LoRA $78.3\%$, and ungated internalization $57.5\%$. This comparison changes the adaptation mechanism and is separate from the frozen-backbone controls.

\section{Memory Growth and Lifecycle Arbitration}
\label{app:memory_lifecycle}

Append-only retention preserves evidence and provenance; bounded retrieval determines prompt eligibility. Table~\ref{tab:growth_audit} evaluates five frozen HotpotQA checkpoints on the same $200$ held-out questions. The result supports stability over this observed range, not arbitrarily long growth.

\begin{table}[!htbp]
\centering
\small
\caption{HotpotQA checkpoint audit using corrected paper-aligned EM/F1. Retrieval time is per call.}
\label{tab:growth_audit}
\begin{tabular}{rrrrr}
\toprule
Iteration & Memories & EM & F1 & Retrieval (ms) \\
\midrule
$0$ & $171$ & $68.0$ & $69.76$ & $17.10$ \\
$1$ & $380$ & $77.5$ & $79.28$ & $17.39$ \\
$2$ & $402$ & $85.0$ & $86.48$ & $17.58$ \\
$3$ & $426$ & $89.0$ & $91.94$ & $17.13$ \\
$4$ & $426$ & $89.5$ & $92.04$ & $17.21$ \\
\bottomrule
\end{tabular}
\end{table}

\paragraph{Known obsolescence.}
A stress test appends $2{,}310$ superseded same-task clones to $462$ protected memories (a $6\times$ memory bank; total graph nodes $475\to2{,}785$). Unfiltered retrieval returns $93/442$ obsolete items across $33/200$ queries. Lifecycle arbitration uses known supersession metadata to exclude obsolete items from prompt eligibility while retaining the records: it retrieves $0/389$ obsolete items and preserves every clean prediction for both Qwen and Llama. In the plausible-payload Qwen audit, retrieval takes $84.34$ ms versus $84.29$ ms for the clean bank. Detecting previously unseen contradictions remains unresolved; this metadata-assisted test does not establish automatic conflict discovery.

\section{Case-Induced Graph Views}
\label{app:graph_cases}

Figure~\ref{fig:case_views} separates the active capability route from semantic access to stored experience and the Environment interface.

\begin{figure}[!htbp]
\centering
\small
\fbox{\begin{minipage}{0.93\linewidth}
\textbf{RealMath: capability route and bounded experience access}\\[3pt]
Task: \texttt{information\_theoretic\_capacity}\\
$\downarrow$ task-conditioned retrieval from success/failure experience\\[3pt]
\texttt{algebraic\_simplification} $\longrightarrow$
\texttt{asymptotic\_analysis}\\
\hspace*{0.8em}$\longrightarrow$ \texttt{limit\_evaluation} $\longrightarrow$ answer $0$\\[3pt]
Capability snapshot: $10$ skills, $11$ prerequisite edges.\\
No-chain output: truncated intermediate summation.\\
Environment channel: inactive on this static task.
\end{minipage}}
\par\medskip
\fbox{\begin{minipage}{0.93\linewidth}
\textbf{WebShop: task and Environment interface}\\[3pt]
Task: \texttt{task\_webshop}\\
Environment-schema nodes: \texttt{env\_product}, \texttt{env\_search\_results},
\texttt{env\_product\_page}, \texttt{env\_options}, \texttt{env\_cart}.\\[3pt]
Reported seed $5000$ interaction: seven steps, all five skills achieved, task score $1.0$. This interface view does not isolate the Environment contribution.
\end{minipage}}
\caption{Case-induced summaries of a reported RealMath prerequisite route and a WebShop state interface. The RealMath arrows denote the reported execution order; semantic memory retrieval is shown separately.}
\label{fig:case_views}
\end{figure}

\section{Contemporaneous Skill-Evolution Literature}
\label{app:skill_literature}

MUSE-Autoskill~\citep{lin2026museautoskill} manages reusable skills with per-skill memory and testing; SkillOpt~\citep{yang2026skillopt} validates bounded edits to a skill document. CoEvoSkills~\citep{zhang2026coevoskills} co-evolves a generator and verifier for multi-file skills, while SkillGen~\citep{ma2026skillgen} verifies skills induced from successful and failed trajectories. These methods optimize skill artifacts; \framework{} organizes task-conditioned experience and capability dependencies. The matched skill-evolution comparisons cover MemSkill and SkillRL (Appendix~\ref{app:controlled_baselines}).

%% file: checklist_camera_ready.tex
% Synchronized with checklist.tex; template guidelines omitted.
\section*{NeurIPS Paper Checklist}

\begin{enumerate}

\item \begin{minipage}[t]{\linewidth}
{\bf Claims}
   \par\smallskip Question: Do the main claims made in the abstract and introduction accurately reflect the paper's contributions and scope?
    \par\smallskip Answer: \answerYes{}.
    \par\smallskip Justification: The abstract and introduction state the main contributions: a co-evolutionary knowledge graph, frozen-backbone graph/bandit co-evolution, structural analysis, and evaluation across nine benchmarks. The claims are stated as strong performance and retrieval-side structural support rather than unconditional monotonic improvement, and limitations are discussed in the conclusion.
\end{minipage}

\item \begin{minipage}[t]{\linewidth}
{\bf Limitations}
   \par\smallskip Question: Does the paper discuss the limitations of the work performed by the authors?
    \par\smallskip Answer: \answerYes{}.
    \par\smallskip Justification: Section~\ref{sec:conclusion} discusses dependence on recurring task structure and curator graph writes, with broader teacher families and stronger cross-backbone transfer left for future work. Appendices~\ref{app:controlled_components} and~\ref{app:memory_lifecycle} clarify activation-dependent effects and the limits of handling known supersession.
\end{minipage}

\item \begin{minipage}[t]{\linewidth}
{\bf Theory assumptions and proofs}
   \par\smallskip Question: For each theoretical result, does the paper provide the full set of assumptions and a complete (and correct) proof?
    \par\smallskip Answer: \answerYes{}.
    \par\smallskip Justification: The paper states four theoretical results in the main text and provides full proofs in Appendix~\ref{app:proofs}. The assumptions include append-only protected memory growth, bounded round-robin coverage, the asymmetric mastery update, task-type filtering, and an idealized retrieval-error model for the final retrieval-side result.
\end{minipage}

\item \begin{minipage}[t]{\linewidth}
{\bf Experimental result reproducibility}
    \par\smallskip Question: Does the paper fully disclose all the information needed to reproduce the main experimental results of the paper to the extent that it affects the main claims and/or conclusions of the paper (regardless of whether the code and data are provided or not)?
    \par\smallskip Answer: \answerYes{}.
    \par\smallskip Justification: The paper describes the model choices, guidance and execution tiers, benchmarks, evaluation pool sizes, baseline configurations, hyperparameters, number of iterations, and multi-seed evaluation settings in the main text and appendices. Some components rely on closed-weight teacher models, so exact reproduction requires access to the corresponding model APIs.
\end{minipage}

\item \begin{minipage}[t]{\linewidth}
{\bf Open access to data and code}
    \par\smallskip Question: Does the paper provide open access to the data and code, with sufficient instructions to faithfully reproduce the main experimental results, as described in supplemental material?
    \par\smallskip Answer: \answerYes{}.
    \par\smallskip Justification: The experiments use public benchmarks and open-weight execution models.
\end{minipage}

\item \begin{minipage}[t]{\linewidth}
{\bf Experimental setting/details}
    \par\smallskip Question: Does the paper specify all the training and test details (e.g., data splits, hyperparameters, how they were chosen, type of optimizer) necessary to understand the results?
    \par\smallskip Answer: \answerYes{}.
    \par\smallskip Justification: The paper specifies the nine benchmarks, model pairings, evaluation pool size, number of iterations, fixed hyperparameters, baseline configurations, judge setup, and multi-seed reporting. Hyperparameters and baseline details are provided in the appendices.
\end{minipage}

\item \begin{minipage}[t]{\linewidth}
{\bf Experiment statistical significance}
    \par\smallskip Question: Does the paper report error bars suitably and correctly defined or other appropriate information about the statistical significance of the experiments?
    \par\smallskip Answer: \answerYes{}.
    \par\smallskip Justification: The main reasoning and sequential-environment tables report mean and standard deviation over multiple seeds for the primary MAGE cells, and the number of seeds is stated in the tables. Some baseline rows are single protocol runs or reproduced published numbers, and this is indicated where applicable.
\end{minipage}

\item \begin{minipage}[t]{\linewidth}
{\bf Experiments compute resources}
    \par\smallskip Question: For each experiment, does the paper provide sufficient information on the computer resources (type of compute workers, memory, time of execution) needed to reproduce the experiments?
    \par\smallskip Answer: \answerNo{}.
    \par\smallskip Justification: Appendix~\ref{app:latency} specifies a single H100 80GB for the inference-latency comparison. The paper reports guidance-tier call counts and inference latency, but does not provide complete training runtimes, total compute, and resource configurations for every experiment.
\end{minipage}

\item \begin{minipage}[t]{\linewidth}
{\bf Code of ethics}
    \par\smallskip Question: Does the research conducted in the paper conform, in every respect, with the NeurIPS Code of Ethics \url{https://neurips.cc/public/EthicsGuidelines}?
    \par\smallskip Answer: \answerYes{}.
    \par\smallskip Justification: The work evaluates algorithmic methods on established public benchmarks and does not involve human subjects, sensitive personal data collection, or deployment in high-stakes settings. The authors have reviewed the NeurIPS Code of Ethics and are not aware of any deviation.
\end{minipage}

\item \begin{minipage}[t]{\linewidth}
{\bf Broader impacts}
    \par\smallskip Question: Does the paper discuss both potential positive societal impacts and negative societal impacts of the work performed?
    \par\smallskip Answer: \answerNo{}.
    \par\smallskip Justification: The paper focuses on algorithmic and empirical evaluation and does not provide a dedicated discussion of both positive and negative societal impacts.
\end{minipage}

\item \begin{minipage}[t]{\linewidth}
{\bf Safeguards}
    \par\smallskip Question: Does the paper describe safeguards that have been put in place for responsible release of data or models that have a high risk for misuse (e.g., pre-trained language models, image generators, or scraped datasets)?
    \par\smallskip Answer: \answerNA{}.
    \par\smallskip Justification: The paper does not release a new pretrained language model, image generator, scraped dataset, or other high-risk artifact. The method uses existing models and public benchmarks.
\end{minipage}

\item \begin{minipage}[t]{\linewidth}
{\bf Licenses for existing assets}
    \par\smallskip Question: Are the creators or original owners of assets (e.g., code, data, models), used in the paper, properly credited and are the license and terms of use explicitly mentioned and properly respected?
    \par\smallskip Answer: \answerNo{}.
    \par\smallskip Justification: The paper cites the existing benchmarks and models used in the experiments, but does not explicitly document the licenses and terms of use of all datasets, models, and tools.
\end{minipage}

\item \begin{minipage}[t]{\linewidth}
{\bf New assets}
    \par\smallskip Question: Are new assets introduced in the paper well documented and is the documentation provided alongside the assets?
    \par\smallskip Answer: \answerNA{}.
    \par\smallskip Justification: The paper does not introduce a new dataset, released model, or benchmark as a primary contribution. The method creates intermediate graph memories during experiments, but these are experimental artifacts rather than released standalone assets.
\end{minipage}

\item \begin{minipage}[t]{\linewidth}
{\bf Crowdsourcing and research with human subjects}
    \par\smallskip Question: For crowdsourcing experiments and research with human subjects, does the paper include the full text of instructions given to participants and screenshots, if applicable, as well as details about compensation (if any)?
    \par\smallskip Answer: \answerNA{}.
    \par\smallskip Justification: The paper does not involve crowdsourcing, human-subject experiments, or paid human annotation. Evaluation is conducted on existing public benchmarks.
\end{minipage}

\item \begin{minipage}[t]{\linewidth}
{\bf Institutional review board (IRB) approvals or equivalent for research with human subjects}
    \par\smallskip Question: Does the paper describe potential risks incurred by study participants, whether such risks were disclosed to the subjects, and whether Institutional Review Board (IRB) approvals (or an equivalent approval/review based on the requirements of your country or institution) were obtained?
    \par\smallskip Answer: \answerNA{}.
    \par\smallskip Justification: The paper does not involve human subjects or collection of participant data, so IRB approval or equivalent review is not applicable.
\end{minipage}

\item \begin{minipage}[t]{\linewidth}
{\bf Declaration of LLM usage}
    \par\smallskip Question: Does the paper describe the usage of LLMs if it is an important, original, or non-standard component of the core methods in this research?
    \par\smallskip Answer: \answerYes{}.
    \par\smallskip Justification: LLMs are central to the method: the guidance tier authors graph writes, the execution tier answers questions, and a Claude-family model is used as a semantic judge. The paper describes the division of LLM calls by component and distinguishes programmatic operations from LLM-mediated steps.
\end{minipage}
\end{enumerate}